\documentclass[11pt,a4paper]{article}
\usepackage{amsfonts}
\usepackage{mathrsfs}
\usepackage{multirow}
\usepackage{cite}
\usepackage{algorithm}
\usepackage{algorithmic}
\usepackage{graphicx}
\usepackage{subfigure}
 \usepackage{epsfig}
 \usepackage{epstopdf}
\usepackage[T1]{fontenc}
\usepackage{geometry}
\usepackage{amsbsy,amsmath,latexsym,amsfonts, epsfig, color, authblk, amssymb, graphics, bm}
\usepackage{epsf,slidesec,epic,eepic}
\usepackage{fancybox}
\usepackage{fancyhdr}
\usepackage{setspace}
\usepackage{nccmath}
\usepackage{cases}
\usepackage[colorlinks, citecolor=blue]{hyperref}

\newtheorem{theorem}{Theorem}[section]
\newtheorem{lemma}{Lemma}[section]
\newtheorem{definition}{Definition}[section]

\newtheorem{proposition}{Proposition}[section]
\newtheorem{corollary}{Corollary}[section]

\newtheorem{remark}{Remark}[section]

\newtheorem{alemma}{Lemma}

\newtheorem{aexample}{Example}

\newenvironment{proof}{{\noindent \bf Proof:}}{\hfill$\Box$\medskip}

\definecolor{lred}{rgb}{1,0.8,0.8}
\definecolor{lblue}{rgb}{0.8,0.8,1}
\definecolor{dred}{rgb}{0.6,0,0}
\definecolor{dblue}{rgb}{0,0,0.5}
\definecolor{dgreen}{rgb}{0,0.5,0.5}

 \title{Equivalent Lipschitz surrogates for zero-norm and rank optimization problems
 \footnote{This work is supported by the National Natural Science Foundation of China under project No. 11571120 and No. 11701186,
 the Natural Science Foundation of Guangdong Province under project No. 2015A030313214 and No. 2017A030310418.}}

\author{Yulan Liu\footnote{Ylliu@gdut.edu.cn. School of Mathematics, GuangDong University of Technology, Guangzhou.},\ \
 Shujun Bi\footnote{Corresponding author(bishj@scut.edu.cn). School of Mathematics, South China University of Technology, Guangzhou.}
 \ \ {\rm and}\ \
 Shaohua Pan\footnote{shhpan@scut.edu.cn. School of Mathematics, South China University of Technology, Guangzhou.}}

 \date{}

\begin{document}

 \maketitle

 \begin{abstract}
  This paper proposes a mechanism to produce equivalent Lipschitz surrogates
  for zero-norm and rank optimization problems by means of the global
  exact penalty for their equivalent mathematical programs with an equilibrium
  constraint (MPECs). Specifically, we reformulate these combinatorial problems
  as equivalent MPECs by the variational characterization of the zero-norm and rank function,
  show that their penalized problems, yielded by moving the equilibrium constraint
  into the objective, are the global exact penalization, and obtain the equivalent
  Lipschitz surrogates by eliminating the dual variable in the global exact penalty.
  These surrogates, including the popular SCAD function in statistics,
  are also difference of two convex functions (D.C.) if the function and constraint set
  involved in zero-norm and rank optimization problems are convex.
  We illustrate an application by designing a multi-stage convex relaxation approach
  to the rank plus zero-norm regularized problem.
 \end{abstract}


 \section{Introduction}\label{sec1}

  This paper concerns with zero-norm and rank optimization problems, which aim at
  seeking a sparse solution or/and a low-rank solution and have a host of applications
  in a variety of fields such as statistics \cite{Tib96,FanLi01,Negahban11},
  signal and image processing \cite{DS89,DL92}, machine learning \cite{Candes09,Recht10},
  control and system identification \cite{Fazel02}, finance \cite{Pietersz04},
  and so on. Due to the combinatorial property of
  the zero-norm and rank function, these problems are generally NP-hard.
  One popular way to deal with them is to use the convex relaxation technique,
  which typically yields a desirable local optimal even feasible solution
  via a single or a sequence of numerically tractable convex optimization problems.

  \medskip

  The $\ell_1$ norm minimization, as a convex relaxation for the zero-norm minimization,
  became popular due to the important results in \cite{DS89,DL92,Tib96}. Among others,
  the results of \cite{DS89,DL92} quantify the ability of the $\ell_1$ norm minimization
  problem to recover sparse reflectivity functions. For brief historical accounts on
  the use of the $\ell_1$ norm minimization in statistics and signal processing;
  please see \cite{Bhlmann11,Trop06}. Later, Fazel \cite{Fazel02} showed that
  the nuclear norm is the convex envelope of the rank function in the unit ball
  on the spectral norm and initiated the research for the nuclear norm convex relaxation.
  In the past ten years, this method received much attention from many fields
  such as information, computer science, statistics, optimization, and so on
  (see, e.g., \cite{Candes09,Recht10,KesMO10,Negahban11,Toh10}),
  and it was shown that a single nuclear norm minimization can recover
  a low-rank matrix in the noiseless setting if the sampling operator has a certain
  restricted isometry property \cite{Recht10}, or can yield a solution satisfying
  a certain error bound in the noisy setting \cite{Candes11,Negahban11}.

  \medskip

  The $\ell_1$ norm or nuclear norm convex relaxation problem,
  as a convex surrogate for zero-norm or rank optimization problems,
  has demonstrated to be successful in encouraging a sparse or low-rank solution,
  but their efficiency is challenged in some circumstances. For example,
  Salakhutdinov and Srebro \cite{Srebro10} showed that when certain rows and/or columns
  are sampled with high probability, the nuclear norm minimization may fail in the sense
  that the number of observations required for recovery is much more than that of
  the uniform sampling, and Negahban and Wainwright \cite{Negahban12} also pointed out
  the influence of such heavy sampling schemes on the recovery error bound.
  In particular, when seeking a sparse (or low-rank) solution from a set which
  has a structure to conflict the role of the $\ell_1$ norm (or nuclear norm)
  to promote sparsity (or low-rank), say, the simplex set \cite{LiRS16},
  the correlation matrix set and the density matrix set \cite{MiaoPS16},
  the $\ell_1$ norm (or nuclear norm) minimization will fail to yielding a sparse
  (or low-rank) solution. The key to bring about this dilemma is the significant
  difference between the convex $\ell_1$ norm (respectively, nuclear norm) and
  the nonconvex zero-norm (respectively, rank function).

  \medskip

  To enhance the solution quality of the $\ell_1$ norm and nuclear norm convex surrogate,
  some researchers pay their attentions to nonconvex surrogates of the zero-norm
  and rank function. Two popular nonconvex surrogates for the zero-norm (respectively,
  rank function) are the $\ell_p\ (0<\!p\!<1)$ norm and logarithm function
  (respectively, the Schattern-$p$ function and logarithmic determinant function).
  Based on these nonconvex surrogates, some sequential convex relaxation algorithms
  were developed (see, e.g., \cite{Candes08,Fazel03,Mohan12,LXW13}) and confirmed to
  have better performance in yielding sparse and low-rank solutions. In addition,
  the folded concave penalty functions such as the SCAD function \cite{FanLi01}
  and the MCP function \cite{Zhang10} are also a class of popular nonconvex surrogates
  for the zero-norm, which are proposed in statistics to correct the bias of
  the $\ell_1$ norm convex surrogate, and some adaptive algorithms were developed
  by using these surrogates (see \cite{Zou06}).

  \medskip

  The existing nonconvex surrogates for the zero-norm and rank function are all
  heuristically constructed, and now it is unclear whether these nonconvex
  surrogates have the same global optimal solution set as zero-norm and
  rank optimization problems do or not. The main contribution of this work
  is to propose a mechanism to produce equivalent Lipschitz surrogates in
  the sense that they have the same global optimal solution set as zero-norm and
  rank optimization problems do, with the help of the global exact penalty for
  their equivalent MPECs. Due to the excellent properties, with this class of
  nonconvex surrogates one may expect to develop more effective convex relaxation algorithms.

  \medskip

  Specifically, we reformulate zero-norm and rank optimization problems
  as equivalent MPECs by the variational characterization of the zero-norm and
  rank function, and show that the penalized problems, yielded by moving
  the equilibrium constraint into the objective, are uniformly partial calm
  over the global optimal solution set under a mild condition.
  The uniform partial calmness over the global optimal solution set,
  extending the partial calmness at a solution point studied by Ye et al.
  \cite{Yezhu95,Yezhu97}, is proved to coincide with the global exact
  penalization of the corresponding penalized problem (see Section \ref{sec2}).
  By eliminating the dual variable in the global exact penalty,
  we achieve the equivalent Lipschitz surrogates. Interestingly, these surrogates
  are also D.C. if the function and constraint set involved in zero-norm and
  rank optimization problems are convex, and the SCAD function
  can be produced by the mechanism (see Example \ref{SCAD} in Appendix B).
  Finally, we illustrate an application of these equivalent surrogates in
  the design of a multi-stage convex relaxation approach to the rank plus
  zero-norm regularized problem.

  \medskip

  It is worthwhile to point out that there are few works on exact penalty for
  the MPEC even the optimization problem involving non-polyhedral conic constraints,
  although much research has been done on exact penalty for classical nonlinear
  programming problems and the MPECs involving polyhedral conic constraints
  (see, e.g., \cite{LPR962,Yezhu95,Yezhu97}). In this work, we establish
  the global exact penalty for equivalent MPECs of group zero-norm
  and rank optimization problems, thereby providing a mechanism to produce
  equivalent Lipschitz surrogates for these combinatorial problems.
  Such global exact penalty was separately obtained in \cite{BiPan14,BiPan17}
  for the zero-norm minimization problem and the rank regularized minimization problem.
  Here, we present a simple unified proof by the uniform partial calmness
  for a large class of zero-norm and rank optimization problems including
  group zero-norm optimization problems, rank plus zero-norm optimization problems,
  and the simultaneous rank and zero-norm minimization problem. In addition,
  we emphasize that although the penalized constraint in our MPECs is D.C.
  (see Section \ref{sec3}-\ref{sec6}), the exact penalty results developed
  for general DC programming in \cite{LeThi12} are not applicable to it.

  \medskip

  Recently, Chen et al. \cite{Chen16} studied exact penalization for the problems
  with a class of nonconvex and non-Lipschitz objective functions, which arise
  from nonconvex surrogates for zero-norm minimization problems. They focused on
  the existence of exact penalty parameters regarding local minimizers,
  stationary points and $\epsilon$-minimizers. However, here we are interested in
  the existence of exact penalty parameter regarding the global optimal solution
  to the equivalent MPECs of zero-norm and rank optimization problems.

  \bigskip
  \noindent
  {\bf Notations.} Let $\mathbb{R}^{n_1\times n_2}\,(n_1\le n_2)$ be the space of all $n_1\!\times\!n_2$ real matrices,
  endowed with the trace inner product and its induced Frobenius norm $\|\cdot\|_F$.
  Let $\mathbb{O}^{n\times\kappa}$ be the set consisting of all $n\times\kappa$ matrices
  whose columns are mutually orthonormal to each other, and write $\mathbb{O}^{n} =\mathbb{O}^{n\times n}$.
  Let $\mathbb{Z}$ be a finite dimensional real vector space equipped with the inner
  product $\langle \cdot,\cdot\rangle$ and its induced norm $\|\cdot\|$. Denote by
  $\mathbb{B}_{\mathbb{Z}}$ the closed unit ball of $\mathbb{Z}$ centered at the origin,
  and by $\mathbb{B}_{\mathbb{Z}}(z,\varepsilon)$ the closed ball of $\mathbb{Z}$
  centered at $z$ of radius $\varepsilon>0$. When the space $\mathbb{Z}$ is known from the context,
  we delete the subscript $\mathbb{Z}$ from $\mathbb{B}_{\mathbb{Z}}$.
  Let $e$ and $E$ be the vector and matrix of all ones whose dimension are known from the context.
  For $x\in\mathbb{R}^n$, $\pi(x)\in\mathbb{R}^n$ is the vector obtained by
  arranging the entries of $|x|$ in a nonincreasing order; for $X\in\mathbb{R}^{n_1\times n_2}$,
  $\pi(X)\in\mathbb{R}^{n_1n_2}$ means the vector obtained by arranging the entries of $|X|$
  in a nonincreasing order; and $\pi_i(\cdot)$ denotes the $i$th entry of $\pi(\cdot)$.
  For $X\in\mathbb{R}^{n_1\times n_2}$, $\sigma(X)\in\mathbb{R}^{n_1}$ means
  the singular value vector of $X$ with entries arranged in a non-increasing order,
  $\|X\|_*$ and $\|X\|$ are the nuclear norm and the spectral norm of $X$, respectively,
  and $\|X\|_\infty$ means the entry $\ell_{\infty}$-norm of $X$. Define
  $\mathfrak{B}\!:=\{Z\in\mathbb{R}^{n_1\times n_2}\ |\ \|Z\|\le 1\}$
  and $\mathscr{B}\!:=\{Z\in\mathbb{R}^{n_1\times n_2}\ |\ \|Z\|_{\infty}\le 1\}$.
  For a given $X\in\mathbb{R}^{n_1\times n_2}$ with the SVD as
  $U[{\rm Diag}(\sigma(X))\ \ 0]V^{\mathbb{T}}$, $U_1$ and $V_1$ are the matrix
  consisting of the first $r={\rm rank}(X)$ columns of $U$ and $V$, respectively,
  and $U_2$ and $V_2$ are the matrix consisting of the last $n_1\!-r$ columns
  and $n_2-r$ columns of $U$ and $V$, respectively.

  \medskip

  Let $\Phi$ be the family of proper lower semi-continuous (lsc) functions
  $\phi\!:\mathbb{R}\to(-\infty,+\infty]$ with ${\rm int}({\rm dom}\,\phi)\supseteq[0,1]$
  which are convex in $[0,1]$ and satisfy the following conditions
  \begin{equation}\label{phi-assump}
   1>t^*:=\mathop{\arg\min}_{0\le t\le 1}\phi(t),\ \phi(t^*)=0\ \ {\rm and}\ \
   \phi(1)=1.
  \end{equation}
  For each $\phi\in\Phi$, let $\psi\!:\mathbb{R}\to(-\infty,+\infty]$ be
  the associated closed proper convex function
  \begin{equation}\label{phi-psi}
            \psi(t):=\!\left\{\!
                 \begin{array}{cl}
                  \phi(t) &\textrm{if}\ t\in [0,1],\\
                   +\infty & \textrm{otherwise};
                 \end{array}\right.
  \end{equation}
  and denote by $\psi^*$ the conjugate of $\psi$, i.e., $\psi^*(s)\!:=\sup_{t\in\mathbb{R}}\{st-\psi(t)\}$.
  Since ${\rm dom}\,\psi=[0,1]$, it is easy to check that ${\rm dom}\,\psi^*=\mathbb{R}$
  and $\psi^*$ is nondecreasing in $\mathbb{R}$. Unless otherwise stated,
  $t_0$ appearing in the subsequent sections is the constant associated to $\phi$
  in Lemma \ref{lemma-phi} of Appendix A. For the examples of $\phi\in\Phi$,
  the reader may refer to Appendix B.

 \section{Uniform partial calmness of optimization problems}\label{sec2}

 Let $\theta\!:\mathbb{Z}\to (-\infty,+\infty]$ be a proper lsc function,
 $h\!:\mathbb{Z}\to\mathbb{R}$ be a continuous function,
 and $\Delta$ be a nonempty closed set of $\mathbb{Z}$.
 This section focuses on the uniform partial calmness of
 \begin{equation*}
  ({\rm MP})\qquad\min_{z\in\mathbb{Z}}\Big\{\theta(z)\!:\ h(z)=0,\,z\in\Delta\Big\}.
 \end{equation*}
 Let $\mathcal{F}$ and $\mathcal{F}^*$ denote the feasible set and the global optimal
 solution set of $({\rm MP})$, respectively, and write the optimal value of $({\rm MP})$
 as $v^*({\rm MP})$. We assume that $\mathcal{F}^*\ne\emptyset$. To introduce the concept
 of partial calmness, we consider the perturbed problem of $({\rm MP})$:
 \begin{equation*}
  ({\rm MP}_{\!\epsilon})\qquad\min_{z\in\mathbb{Z}}\Big\{\theta(z)\!:\ h(z)=\epsilon,\,z\in\Delta\Big\}.
 \end{equation*}
 For any given $\epsilon\in\mathbb{R}$, we denote by $\mathcal{F}_{\!\epsilon}$
 the feasible set of $({\rm MP}_{\!\epsilon})$ associated to $\epsilon$.
 \begin{definition}\label{def-pcalm0}(see \cite[Definition 3.1]{Yezhu95} or \cite[Definition 2.1]{Yezhu97})
  The problem $({\rm MP})$ is said to be partially calm at a solution point $z^*$
  if there exist $\varepsilon\!>0$ and $\mu>0$ such that for all $\epsilon\in[-\varepsilon,\varepsilon]$
  and all $z\in (z^*\!+\varepsilon\mathbb{B}_{\mathbb{Z}})\cap\mathcal{F}_{\!\epsilon}$,
  one has
  \(
   \theta(z)-\theta(z^*)+\mu |h(z)|\geq 0.
  \)
 \end{definition}

  The calmness of a mathematical programming problem at a solution point was originally
  introduced by Clarke \cite{Clarke83}, which was later extended to the partial calmness
  at a solution point by Ye and Zhu \cite{Yezhu95,Yezhu97}. Next we strengthen the partial
  calmness of $({\rm MP})$ at a solution point as the partial calmness over
  its global optimal solution set $\mathcal{F}^*$.
  \begin{definition}\label{def-pcalm}
  The problem $({\rm MP})$ is said to be partially calm over its global optimal
  solution set $\mathcal{F}^*$ if it is partially calm at each $z^*\!\in\!\mathcal{F}^*$;
  and it is said to be uniformly partial calm over $\mathcal{F}^*$ if there exists
  $\mu>0$ such that for any $z\in\Delta$,
  \[
    \theta(z)-v^*({\rm MP})+\mu|h(z)|\ge 0.
  \]
  \end{definition}

 It is worthwhile to emphasize that the partial calmness over $\mathcal{F}^*$
 along with the boundedness of $\mathcal{F}^*$ does not imply the uniform partial calmness
 over $\mathcal{F}^*$. In addition, the partial calmness depends on the structure of a problem.
 Equivalent problems may not share the partial calmness simultaneously;
 for example, for the following equivalent form of $({\rm MP})$
  \begin{equation}\label{EMP}
  \min_{z\in\mathbb{Z}}\big\{\theta(z)\!:\ {\rm dist}(z,\mathcal{F})=0\big\},
 \end{equation}
 it is easy to verify that the local Lipschitz of $\theta$ relative to $\mathcal{F}^*$
 is enough for the partial calmness of \eqref{EMP} over $\mathcal{F}^*$,
 but it may not guarantee that of $({\rm MP})$ over $\mathcal{F}^*$. Define
 \begin{equation}\label{Gamma-map}
  \Gamma(\epsilon):=\left\{z\in\Delta\ |\ h(z)=\epsilon\right\}\quad{\rm for}\ \epsilon\in\mathbb{R}.
 \end{equation}
 The following lemma states that under a suitable condition for $\theta$,
 the partial calmness of $({\rm MP})$ over $\mathcal{F}^*$ is implied by
 the calmness of the multifunction $\Gamma$ at $0$ for each $z\in\mathcal{F}^*$.
 The proof is similar to that of \cite[Lemma 3.1]{Ye97}, and we include it for completeness.
 \begin{lemma}\label{lemma-calm}
  Suppose that $\theta$ is locally Lipschitzian relative to $\Delta$.
  If the multifunction $\Gamma$ is calm at $0$ for any $z\in\mathcal{F}^*$,
  then the problem $({\rm MP})$ is partially calm over $\mathcal{F}^*$.
 \end{lemma}
 \begin{proof}
  Let $z^*$ be an arbitrary point from $\mathcal{F}^*$. Since $\theta$ is
  locally Lipschitzian relative to $\Delta$ and $z^*\in\Delta$,
  there exist $\varepsilon'>0$ and $L_{\theta}>0$ such that
  for any $z',z''\in\mathbb{B}(z^*,\varepsilon')\cap\Delta$,
  \begin{equation}\label{temp-equa21}
   |\theta(z')-\theta(z'')|\le L_{\theta}\|z'-z''\|.
  \end{equation}
  In addition, since the multifunction $\Gamma$ is calm at $0$ for $z^*$,
  by invoking \cite[Exercise 3H.4]{DR09}, there exist constants $\nu>0$
  and $\delta'>0$ such that for all $\omega\in\mathbb{R}$,
  \begin{equation*}
   \Gamma(\omega)\cap\mathbb{B}(z^*,\delta')\subseteq\Gamma(0)+\nu|\omega|\mathbb{B}_{\mathbb{Z}}.
  \end{equation*}
  Set $\varepsilon=\min(\varepsilon',\delta')/2$. Let $\epsilon$ be an arbitrary point from $[-\varepsilon,\varepsilon]$
  and $z$ be an arbitrary point from $(z^*\!+\varepsilon\mathbb{B}_{\mathbb{Z}})\cap\mathcal{F}_{\!\epsilon}$.
  Clearly, $z\in\Gamma(\epsilon)\cap\mathbb{B}(z^*,\delta')$.
  Applying the last inclusion with $\omega=\epsilon$, we obtain that ${\rm dist}(z,\Gamma(0))\le\nu|\epsilon|=\nu|h(z)|$.
  From the closedness of $\Gamma(0)$, there exists a point $\widehat{z}\in\Gamma(0)$ such that
 \(
    {\rm dist}(z,\Gamma(0))=\|z-\widehat{z}\|\le\nu|h(z)|.
  \)
  Notice that $\|\widehat{z}-z^*\|=\|\widehat{z}-z+z-z^*\|\le 2\|z-z^*\|\le\varepsilon'$.
  Together with \eqref{temp-equa21} and $z,\widehat{z}\in\Delta$, we have
  \[
   \theta(z^*)\le \theta(\widehat{z})=\theta(z)-\theta(z)+\theta(\widehat{z})
   \le \theta(z)+L_{\theta}\|z-\widehat{z}\|\le \theta(z)+L_{\theta}\nu|h(z)|,
  \]
  where the first inequality is by the feasibility of $\widehat{z}$ and the optimality
  of $z^*$ to $({\rm MP})$. The last inequality and the arbitrariness of $z^*$ in $\mathcal{F}^*$
  implies the desired conclusion.
 \end{proof}

  Next we shall establish the relation between the uniform partial calmness of $({\rm MP})$
  over $\mathcal{F}^*$ and the global exact penalization of the following penalized problem:
  \begin{equation*}
  ({\rm EPMP_{\!\mu}})\qquad\ \min_{z\in\mathbb{Z}}\Big\{\theta(z)+\mu|h(z)|\!:\ z\in\Delta\Big\}.
 \end{equation*}
  In \cite[Proposition 2.2]{Yezhu97}, Ye et al. showed that under the continuity of $\theta$,
  the partial calmness of $({\rm MP})$ at a local minimum is equivalent to the local exact
  penalization of $({\rm EPMP_{\!\mu}})$. Here, we extend this result and show that the uniform
  partial calmness of $({\rm MP})$ over $\mathcal{F}^*$ is equivalent to the global exact
  penalization of $({\rm EPMP_{\mu}})$, i.e., there exists $\overline{\mu}>0$ such that
  the global optimal solution set of each $({\rm EPMP_{\!\mu}})$ associated to
  $\mu\!>\overline{\mu}$ coincides with that of $({\rm MP})$,
  where $\overline{\mu}$ is called the threshold of the exact penalty.
 \begin{proposition}\label{Epenalty-prop}
  For the problems $({\rm MP})$ and $({\rm EPMP_{\!\mu}})$, the following statements hold.
 \begin{itemize}
   \item [(a)] The problem $({\rm MP})$ is uniformly partial calm over its global
               optimal solution set $\mathcal{F}^*$ if and only if the problem $({\rm EPMP_{\!\mu}})$ is
               a global exact penalty of $({\rm MP})$.

  \item[(b)]  Suppose that the function $\theta$ is coercive or the set $\Delta$ is compact.
              Then, the partial calmness of $({\rm MP})$ over $\mathcal{F}^*$
              implies the global exact penalization of $({\rm EPMP_{\!\mu}})$.
  \end{itemize}
 \end{proposition}
 \begin{proof}
 We denote $\mathcal{F}_{\!\mu}^*$ by the global optimal solution set of $({\rm EPMP_{\!\mu}})$
 associated to $\mu>0$.

 \medskip
 \noindent
 {\bf(a)} ``$\Longleftarrow$''. Since the problem $({\rm EPMP_{\!\mu}})$ is a global exact penalty
 of $({\rm MP})$, there exists a constant $\overline{\mu}>0$ such that for any $\mu>\overline{\mu}$,
 $\mathcal{F}_{\!\mu}^*=\mathcal{F}^*$. Take an arbitrary point $z^*\in\mathcal{F}^*$. Then,
 for any $\gamma>0$, $z^*$ is also a global optimal solution of $({\rm EPMP_{\overline{\mu}+\gamma}})$.
 Thus, for all $z\in\Delta$, from the feasibility of $z$ and the optimality of $z^*$
 to $({\rm EPMP_{\overline{\mu}+\gamma}})$,
 \[
   \theta(z)+(\overline{\mu}+\gamma)|h(z)|\ge \theta(z^*)+(\overline{\mu}+\gamma)|h(z^*)|,
 \]
 which is equivalent to saying that
 \(
   \theta(z)-v^*({\rm MP})+(\overline{\mu}+\!\gamma)|h(z)|\ge(\overline{\mu}+\!\gamma)|h(z^*)|\ge 0.
 \)
 Taking the limit $\gamma\downarrow 0$ to this inequality, we obtain
 $\theta(z)-v^*({\rm MP})+\overline{\mu}|h(z)|\ge0$. This shows that
 $({\rm MP})$ is uniformly partially calm over its optimal solution set $\mathcal{F}^*$.

 \medskip
 \noindent
 ``$\Longrightarrow$''. Since the problem $({\rm MP})$ is uniformly partial calm over
 its global optimal solution set $\mathcal{F}^*$, there exists a constant $\widehat{\mu}>0$ such that
 for all $z\in\Delta$,
 \[
   \theta(z)-v^*({\rm MP})+\widehat{\mu}|h(z)|\ge 0.
 \]
 We first prove that for any $\mu\ge\widehat{\mu}$, $\mathcal{F}^*\subseteq\mathcal{F}_{\!\mu}^*$.
 Let $z^*$ be an arbitrary point from $\mathcal{F}^*$. Fix an arbitrary $\mu\ge\widehat{\mu}$.
 From the last inequality, it follows that for any $z\in\Delta$,
 \[
   \theta(z)+\mu|h(z)|\ge\theta(z)+\widehat{\mu}|h(z)|\ge\theta(z^*)=\theta(z^*)+\mu|h(z^*)|.
 \]
 This, by the arbitrariness of $z\in\Delta$, implies that $z^*\in\mathcal{F}_{\!\mu}^*$.
 Consequently, for any $\mu>\widehat{\mu}$, it holds that $\mathcal{F}^*\subseteq\mathcal{F}_{\!\mu}^*$.
 Next we shall prove that for any $\mu>\widehat{\mu}$,
 $\mathcal{F}_{\!\mu}^*\subseteq \mathcal{F}^*$. To this end, fix an arbitrary $\mu>\widehat{\mu}$
 and take an arbitrary point $z_{\mu}\in\mathcal{F}_{\!\mu}^*$. Let $z^*\in\mathcal{F}^*$. Then,
 \[
   \theta(z_{\mu})+\mu|h(z_{\mu})|
   \le \theta(z^*)+\mu|h(z^*)|=v^*({\rm MP})+\frac{\mu+\widehat{\mu}}{2}|h(z^*)|
   \le \theta(z_{\mu})+\frac{\mu+\widehat{\mu}}{2}|h(z_{\mu})|
 \]
 where the first inequality is by the optimality of $z_{\mu}$ and the feasibility of $z^*$
 to $({\rm EPMP_{\!\mu}})$, and the second one is due to
 $\mathcal{F}^*\subseteq\mathcal{F}_{\mu'}^*$ for $\mu'\!=\frac{1}{2}(\mu+\widehat{\mu})$,
 implied by the above arguments. The last inequality implies
 $\frac{1}{2}(\mu-\widehat{\mu})|h(z_{\mu})|\le 0$, and then $h(z_{\mu})=0$.
 This shows that $z_{\mu}$ is feasible to the problem $({\rm MP})$. Together with
 the first inequality in the last equation, $z_{\mu}$ is optimal to $({\rm MP})$.
 The stated inclusion follows by the arbitrariness of $z_{\mu}$ in $\mathcal{F}_{\!\mu}^*$.

 \medskip
 \noindent
 {\bf(b)} Since $\theta$ is coercive or the set $\Delta$ is compact,
 for each $\mu>0$ we have $\mathcal{F}_{\!\mu}^*\ne\emptyset$.
 Assume that $({\rm MP})$ is partially calm over $\mathcal{F}^*$.
 To prove that $({\rm EPMP_{\!\mu}})$ is a global exact penalty for $({\rm MP})$,
  we first argue that there exists $\mu^*\!>0$ such that for any $\mu>\mu^*$,
  $\mathcal{F}^*\subseteq\!\mathcal{F}_{\!\mu}^*$. If not, for each sufficiently large
  $k$, there exist $z^{k,*}\in\mathcal{F}^*$ and $z^k\in\Delta$ such that
  \begin{equation}\label{EPeq1}
   \theta(z^k)+k|h(z^k)|<\theta(z^{k,*})+k|h(z^{k,*})|=v^*({\rm MP}).
  \end{equation}
  If $\Delta$ is compact, clearly, $\{z^k\}$ is bounded.
  If $\theta$ is coercive, inequality \eqref{EPeq1} implies that $\{z^k\}$ is also bounded.
  Thus, from $z^k\in\Delta$ and the closedness of $\Delta$, we assume
  (if necessary taking a subsequence) that $z^k\to\overline{z}\in\Delta$.
  Notice that \eqref{EPeq1} can be equivalently written as
  \begin{equation*}
   0\le |h(z^k)|<\frac{1}{k}\big[v^*({\rm MP})-\theta(z^k)\big].
  \end{equation*}
 Take $k\!\to\!\infty$ to the both sides. By the continuity of $h$ and
 the lower semi-continuity of $\theta$,
 \[
   0\le|h(\overline{z})|=\lim_{k\to+\infty}|h(z^k)|\le\lim_{k\to+\infty}\frac{1}{k}\big[v^*({\rm MP})-\theta(z^k)\big]=0.
 \]
 In addition, from \eqref{EPeq1} it follows that $\theta(\overline{z})\le v^*({\rm MP})$.
 This shows that $\overline{z}$ is a global optimal solution to $({\rm MP})$.
 But then inequality \eqref{EPeq1} gives a contradiction to the partial calmness of $({\rm MP})$
 at $\overline{z}$, which is implied by the given assumption that $({\rm MP})$ is partially calm
 over $\mathcal{F}^*$. Thus, there exists $\mu^*>0$ such that for any $\mu>\mu^*$,
 $\mathcal{F}^*\subseteq\!\mathcal{F}_{\mu}^*$. In addition, using the same arguments as
 those for the direction ``$\Longrightarrow$'' in part (a),
 one may prove that for any $\mu>\mu^*$, $\mathcal{F}_{\!\mu}^*\subseteq\mathcal{F}^*$.
 Thus, $({\rm EPMP_{\!\mu}})$ is a global exact penalty of $({\rm MP})$.
\end{proof}
 \begin{remark}\label{remark-sec2}
  Proposition \ref{Epenalty-prop} show that under the coerciveness of $\theta$
  or the compactness of $\Delta$, the partial calmness of $({\rm MP})$ over $\mathcal{F}^*$,
  the uniformly partial calmness of $({\rm MP})$ over $\mathcal{F}^*$ and
  the global exact penalization of $({\rm EPMP_{\mu}})$ are equivalent each other.
 \end{remark}

 \section{Equivalent L-surrogates of group zero-norm problems}\label{sec3}

  Let $\mathcal{J}=\{J_1,\ldots,J_m\}$ be a partition of $\{1,2,\ldots,n\}$.
  For any given $p\in[1,+\infty]$, define
  \begin{equation*}
   G_{\!\mathcal{J}\!,p}(x):=\big(\|x_{\!_{J_1}}\|_p,\|x_{\!_{J_2}}\|_p,\ldots,\|x_{\!_{J_m}}\|_p\big)^{\mathbb{T}}
   \quad{\rm for}\ \ x\in\mathbb{R}^n.
  \end{equation*}
  The number of nonzero components in $G_{\!\mathcal{J}\!,p}(x)$, denoted by $\|G_{\!\mathcal{J}\!,p}(x)\|_0$,
  is called the group zero-norm of $x$ induced by the partition $\mathcal{J}$ and the $\ell_p$ norm $\|\cdot\|_p$.
  Clearly, when $m=n$ and $J_i=\{i\}$ for $i=1,2,\ldots,m$, $\|G_{\!\mathcal{J}\!,p}(x)\|_0$ reduces to
  the zero-norm $\|x\|_0$ of $x$. As a producer of structured sparsity, the group zero-norm has
  a wide application in statistics, signal and image processing,
  machine learning, and bioinformatics (see, e.g., \cite{YuanLin06,Bach08,Usman11}).
  By the definition of the function family $\Phi$, for any $x\in\mathbb{R}^n$, with $\phi\in\Phi$ one has that
  \begin{equation}\label{Gzn-chara}
   \|G_{\!\mathcal{J}\!,p}(x)\|_0=\!\min_{w\in\mathbb{R}^m}\!
   \Big\{{\textstyle\sum_{i=1}^m}\phi(w_i)\!: \langle e\!-\!w,G_{\!\mathcal{J}\!,p}(x)\rangle=0,\,0\le w\le e\Big\},
  \end{equation}
  that is, the group zero-norm is an optimal value function of a parameterized problem.
 \subsection{Group zero-norm minimization problems}\label{sec3.1}

  Let $f\!:\mathbb{R}^n\to(-\infty,+\infty]$ be a proper lsc function, and
  let $\Omega\subseteq\mathbb{R}^n$ be a closed set. This subsection is devoted
  itself to the following group zero-norm  minimization problem
  \begin{equation}\label{Gzmin}
   \min_{x\in\mathbb{R}^{n}}\!\Big\{\|G_{\!\mathcal{J}\!,p}(x)\|_0\!: f(x)\le\delta,\, x\in\Omega \Big\}
  \end{equation}
  where $\delta\ge 0$ is a constant to represent the noise level. We assume that \eqref{Gzmin}
  has a nonempty global optimal solution set and a nonzero optimal value, denoted by $s^*$.
  Let $\mathcal{S}$ denote the feasible set of \eqref{Gzmin}. From equation \eqref{Gzn-chara},
  it is immediate to obtain the following result.
 \begin{lemma}\label{Gzn-lemma1}
  Let $\phi\!\in\!\Phi$. The group zero-norm minimization problem \eqref{Gzmin} is equivalent to
 \begin{equation}\label{Gzn-equiv}
  \min_{x\in\mathbb{R}^n,w\in\mathbb{R}^{m}}\!
  \Big\{{\textstyle\sum_{i=1}^m}\,\phi(w_i)\!: \langle e\!-\!w,G_{\!\mathcal{J}\!,p}(x)\rangle=0,\,0\le w\le e,\,x\in\Omega,\,f(x)\le\delta\Big\}
 \end{equation}
 in the sense that if $x^*$ is a global optimal solution of \eqref{Gzmin},
 then $(x^*\!,\max({\rm sgn}(G_{\!\mathcal{J}\!,p}(x^*)),t^*e))$
 is globally optimal to \eqref{Gzn-equiv} with the optimal value equal to $s^*$;
 conversely, if $(x^*\!,w^*)$ is a global optimal solution to \eqref{Gzn-equiv},
 then $x^*$ is globally optimal to \eqref{Gzmin}.
 \end{lemma}

  Observe that the minimization problem in \eqref{Gzn-equiv} involves an equilibrium constraint
  \[
   \langle e\!-\!w,G_{\!\mathcal{J}\!,p}(x)\rangle=0,\ e\!-\!w\ge 0,\ G_{\!\mathcal{J}\!,p}(x)\ge 0.
  \]
 Lemma \ref{Gzn-lemma1} shows that the (group) zero-norm minimization problem
 is essentially an MPEC. Such an equivalent reformulation was employed in \cite{BiPan14}
 to develop a penalty decomposition method for zero-norm minimization problems,
 and used in \cite{Feng15} to study the stationary point conditions for
 zero-norm optimization problems. Next we shall establish
 the uniform partial calmness of this MPEC over its global optimal solution set.
 \begin{theorem}\label{Gnz-theorem}
  Let $\phi\in\Phi$. Suppose that there exists $\alpha>0$ such that
  $\pi_{s^*}(G_{\!\mathcal{J}\!,p}(x))\ge\alpha$ for all $x\in\!\mathcal{S}$.
  Then \eqref{Gzn-equiv} is uniformly partial calm over its global optimal solution set and
  \begin{equation}\label{Gzn-epenalty}
  \min_{x\in\mathbb{R}^n,w\in\mathbb{R}^{m}}
  \Big\{{\textstyle\sum_{i=1}^m}\,\phi(w_i)+\varrho\langle e\!-\!w,G_{\!\mathcal{J}\!,p}(x)\rangle\!:\ f(x)\le\delta,\,x\in\Omega,\,0\le w\le e\Big\}
 \end{equation}
  is a global exact penalty for the MPEC \eqref{Gzn-equiv} with threshold $\overline{\varrho}={\phi'_-(1)}/{\alpha}$.
 \end{theorem}
 \begin{proof}
  Let $(x,w)$ be an arbitrary feasible point from $\mathcal{S}\times[0,e]$. Then, it holds that
  \begin{align*}
  {\textstyle\sum_{i=1}^m}\phi(w_i)+\overline{\varrho}\langle e\!-\!w,G_{\!\mathcal{J}\!,p}(x)\rangle
  &\ge {\textstyle\sum_{i=1}^m}\phi(\pi_i(w))+\overline{\varrho}\,{\textstyle\sum_{i=1}^m}\pi_i(G_{\!\mathcal{J}\!,p}(x))(1-\pi_i(w))\\
  &\ge {\textstyle\sum_{i=1}^{s^*}}\big[\phi(\pi_i(w))+\overline{\varrho}\,\pi_{s^*}(G_{\!\mathcal{J}\!,p}(x))(1-\pi_i(w))\big]\nonumber\\
  &\ge {\textstyle\sum_{i=1}^{s^*}}\big[\phi(\pi_i(w))+\phi'_{-}(1)(1-\pi_i(w))\big]\\
  &\ge s^*\phi(1)=s^*,
 \end{align*}
  where the first inequality is using
  $\langle w,G_{\!\mathcal{J}\!,p}(x)\rangle\le\langle\pi(w),\pi(G_{\!\mathcal{J}\!,p}(x))\rangle$,
  the third one is by $\pi_{s^*}(G_{\!\mathcal{J}\!,p}(x))\!\ge\alpha$ and
  $\overline{\varrho}={\phi'_-(1)}/{\alpha}$, and the last one is due to
  $\phi(t)\ge\phi(1)+\phi_{-}'(1)(t-1)$ for $t\in[0,1]$ implied by
  the convexity of $\phi$ in $[0,1]$. Since $s^*$ is the optimal value of \eqref{Gzn-equiv}
  by Lemma \ref{Gzn-lemma1}, the last inequality along with the arbitrariness of $(x,w)$
  in $\mathcal{S}\times[0,e]$ shows that \eqref{Gzn-equiv} is uniformly partial calm
  over its global optimal solution set, which by Proposition \ref{Epenalty-prop}(a) is equivalent to
  saying that \eqref{Gzn-epenalty} is a global exact penalty.
  \end{proof}

  With the function $\psi$ in \eqref{phi-psi} associated to $\phi$ and its conjugate $\psi^*$,
  we can represent the dual variable $w$ in \eqref{Gzn-epenalty} by the variable $x$,
  and obtain the following conclusion.
 \begin{corollary}\label{Gnz-corollary}
  Let $\phi\in\Phi$. Under the assumption of Theorem \ref{Gnz-theorem},
  the problem \eqref{Gzmin} has the same global optimal solution set
  as the following problem with $\varrho>{\phi'_-(1)}/{\alpha}$ does:
 \begin{equation}\label{Gnz-surrogate}
  \min_{x\in\mathbb{R}^n}
  \Big\{\varrho{\textstyle\sum_{i=1}^m}\|x_{\!_{J_i}}\|_p
        -{\textstyle\sum_{i=1}^m}\psi^*(\varrho\|x_{\!_{J_i}}\|_p):\ f(x)\le\delta,\,x\in\Omega\Big\}.
 \end{equation}
 \end{corollary}

 Notice that $\psi^*$ is nondecreasing and convex in $\mathbb{R}_{+}$. So,
 the function $\sum_{i=1}^m\psi^*(\varrho\|x_{\!_{J_i}}\|_p)$ is convex in $\mathbb{R}^n$.
 Thus, the objective function of \eqref{Gnz-surrogate} is locally Lipschitz
 in $\mathbb{R}^n$ by \cite[Theorem 10.4]{Roc70} and provides a class of equivalent
 Lipschitz surrogates for the group zero-norm problem \eqref{Gzmin}.
 If the feasible set of \eqref{Gzmin} is convex, it also provides a class of
 equivalent D.C. surrogates since its objective function is now the difference
 of two convex functions.

 \medskip

 To close this subsection,  we show that the assumption of Theorem \ref{Gnz-theorem} is very mild.
 \begin{lemma}\label{lemma-cond}
  Suppose that $\mathcal{S}$ is bounded, or $\Omega=\mathbb{R}^n$ and $f(x)\equiv g(Ax\!-\!b)$
  for a proper lsc coercive function $g\!:\mathbb{R}^{N}\to(-\infty,+\infty]$
  and a matrix $A\in\mathbb{R}^{N\times n}$.
  Then there exists $\alpha>0$ such that $\pi_{s^*}(G_{\!\mathcal{J}\!,p}(x))\ge\alpha$
  for all $x\in\!\mathcal{S}$.
 \end{lemma}
 \begin{proof}
  Suppose the conclusion does not hold. Then there exist $\{\alpha_k\}\subseteq\mathbb{R}_{+}$
  with $\alpha_k\to 0$ and $\{x^k\}\subseteq\mathcal{S}$ such that
  $\pi_{s^*}(G_{\!\mathcal{J}\!,p}(x^k))\le\alpha_k$. We proceed the arguments by two cases.

  \medskip
  \noindent
  {\bf Case 1: $\mathcal{S}$ is bounded.} Now we may assume (if necessary taking a subsequence)
  that $x^k\to\overline{x}$. Together with $\pi_{s^*}(G_{\!\mathcal{J}\!,p}(x^k))\le\alpha_k$
  and the continuity of $\pi_{s^*}(G_{\!\mathcal{J}\!,p}(\cdot))$, it follows that
  $\pi_{s^*}(G_{\!\mathcal{J}\!,p}(\overline{x}))=0$. Notice that $\overline{x}\in\mathcal{S}$.
  This means that $\|G_{\!\mathcal{J}\!,p}(\overline{x})\|_0=s^*\!-\!1$, which contradicts
  the fact that $s^*$ is the optimal value of the problem \eqref{Gzmin}.

  \medskip
  \noindent
  {\bf Case 2: $\Omega=\mathbb{R}^n$ and $f(x)\equiv g(Ax-b)$.} Now since $\pi_{s^*}(G_{\!\mathcal{J}\!,p}(x^k))\to 0$,
  we may assume (if necessary taking a subsequence) that there exist $i_1,\ldots,i_{m-s^*+1}\in\{1,\ldots,m\}$
  such that $x_{\!J}^k\to 0$ with $J=\bigcup_{k=1}^{m-s^*+1}J_{i_k}$. Notice that
  $g(Ax^k\!-\!b)\le \delta$. By rearranging the columns of $A$ if necessary,
  for each $k$ there exists $z^k\in\mathbb{R}^{n-|J|}$ such that
  \(
    g\big([A_1\quad A_2](z^k;x_{\!J}^k)-b\big)\le \delta,
 \)
 where $[A_1\ \ A_2]=A$ with $A_1\!\in\mathbb{R}^{N\times(n-|J|)}$
 and $A_2\!\in\mathbb{R}^{N\times|J|}$. Write $A_1z^k\!=b-A_2x_{\!J}^k+\xi^k$
 for some $\{\xi^k\}$ with $g(\xi^k)\le \delta$. Since $g$ is lower semi-continuous
 and coercive, the sequence $\{\xi^k\}$ is bounded. Without loss of generality,
 we assume $\xi^k\to\overline{\xi}$ for some $\overline{\xi}$ with $g(\overline{\xi})\le \delta$.
 Since the set $A_1(\mathbb{R}^{N\times(n-|J|)})$ is closed by \cite[Proposition 2.41]{BS00},
 there exists $\overline{z}\in\mathbb{R}^{N\times(n-|J|)}$ such that
 $A_1\overline{z}=b+\overline{\xi}$, and then $g(A(\overline{z}; 0)-b)\le \delta$
 or $(\overline{z};0)\in\mathcal{S}$. Since $\|G_{\!\mathcal{J}\!,p}(\overline{z};0)\|_0\le s^*\!-\!1$,
 we obtain a contradiction to the fact that $s^*$ is the optimal value of \eqref{Gzmin}.
 \end{proof}
 \begin{remark}\label{Gzn-remark}
  When $\Omega=\mathbb{R}^n$ and $f(x)\equiv g(Ax-b)$, the coerciveness of $g$
  does not imply the boundness of $\mathcal{S}$.
  The conclusions of Theorem \ref{Gnz-theorem} and Lemma \ref{lemma-cond} extend
  the exact penalty result in \cite[Theorem 3.3]{BiPan14}. In fact, when taking
  $\phi(t)\equiv t$, $g(\cdot)\!=\|\cdot\|_2,m=n$ and $J_i=\{i\}$ for $i=1,\ldots,m$,
  one can recover the exact penalty result in \cite[Theorem 3.3]{BiPan14}.
 \end{remark}

 \subsection{Group zero-norm regularized problems}\label{sec3.2}

  This subsection is devoted itself to the group zero-norm regularized minimization problem
  \begin{equation}\label{Gzrmin}
   \min_{x\in\mathbb{R}^{n}}\!\Big\{\nu f(x)+\|G_{\!\mathcal{J}\!,p}(x)\|_0\!:\ x\in\Omega \Big\},
  \end{equation}
  where $\nu>0$ is the regularization parameter, and $f$ and $\Omega$ are same as those in Subsection \ref{sec3.1}.
  Assume that \eqref{Gzrmin} has a nonempty global solution set and write its optimal value
  as $\varpi^*$. By the characterization of the group zero-norm in \eqref{Gzn-chara},
  the following result holds.
 \begin{lemma}\label{Gzr-lemma}
  Let $\phi\in\Phi$. The group zero-norm regularized problem \eqref{Gzrmin} is equivalent to
 \begin{equation}\label{Gzr-equiv}
  \min_{x\in\mathbb{R}^n,w\in\mathbb{R}^{m}}
  \Big\{\nu f(x)+\!{\textstyle\sum_{i=1}^m}\,\phi(w_i)\!:\ \langle e\!-\!w,G_{\!\mathcal{J}\!,p}(x)\rangle=0,\,0\le w\le e,\,x\in\Omega\Big\},
 \end{equation}
 in the sense that if $x^*$ is globally optimal to \eqref{Gzrmin},
 then $(x^*\!,\max({\rm sgn}(G_{\!\mathcal{J}\!,p}(x^*)),t^*e))$ is a global
 optimal solution of \eqref{Gzr-equiv} with optimal value equal to $\varpi^*$;
 conversely, if $(x^*,w^*)$ is a global optimal solution of \eqref{Gzr-equiv},
 then $x^*$ is globally optimal to \eqref{Gzrmin}.
 \end{lemma}

 Lemma \ref{Gzr-lemma} states that the group zero-norm regularized problem
 is also an MPEC. Next, under a suitable restriction on $\Omega$,
 we show that the MPEC \eqref{Gzr-equiv} is uniformly partial calm in its global solution set.
 To this end, for any $x$ and $\varrho\!>0$, with $\phi\in\Phi$ define
 \begin{equation}\label{xrho}
  x_j^{\varrho}:=\left\{\begin{array}{cl}
                   0 & {\rm if}\ j\in\bigcup_{i\notin{\rm supp}(y(x,\varrho))}J_i,\\
                 x_j &{\rm otherwise}
                 \end{array}\right.\ \  {\rm for}\ j=1,2,\ldots,n
 \end{equation}
 with
 \begin{equation}\label{yrho}
   y_i(x,\varrho):=\left\{\begin{array}{cl}
                   \!\|x_{\!_{J_i}}\|_p & {\rm if}\ \varrho\|x_{\!_{J_i}}\|_p>\phi_{-}'(1),\\
                       0 &{\rm otherwise}
                 \end{array}\right.\ \  {\rm for}\ i=1,2,\ldots,m.
 \end{equation}

 \vspace{-0.2cm}
 \begin{theorem}\label{Grz-theorem2}
  Let $\phi\in\Phi$. Suppose that $f$ is Lipschitzian relative to $\Omega$
  with constant $L_{\!f}>0$, and for any given $x\in\Omega$ and $\varrho>0$,
  the vector $x^{\varrho}$ lies in $\Omega$. Then, the MPEC \eqref{Gzr-equiv}
  is uniformly partial calm in its global optimal solution set, or equivalently
  the problem
 \begin{equation}\label{Gzr-epenalty}
  \min_{x\in\mathbb{R}^n,w\in\mathbb{R}^{m}}
  \Big\{\nu f(x)+{\textstyle\sum_{i=1}^m}\,\phi(w_i)+\varrho\langle e\!-\!w,G_{\!\mathcal{J}\!,p}(x)\rangle\!:\ x\in\Omega,\,0\le w\le e\Big\}
 \end{equation}
  is a global exact penalty of \eqref{Gzr-equiv} with threshold
  $\overline{\varrho}=\frac{\phi_{-}'(1)(1-t^*)\beta\nu L_{\!f}}{1-t_0}$,
  where $t_0\in[0,1)$ is the constant in Lemma \ref{lemma-phi} of Appendix A
  and $\beta\!=\max\Big(1,\max_{1\le i\le m}|J_i|^{\frac{p-2}{2p}}\Big)$.
 \end{theorem}
 \begin{proof}
  Let $(x,w)$ be an arbitrary point in $\Omega\times\![0,e]$. Define the following index sets
  \[
    I=\Big\{i\!:\ \frac{1}{1-t^*}\le\overline{\varrho}\|x_{\!_{J_i}}\|_p\le\phi_{-}'(1)\Big\}
    \ \ {\rm and}\ \
    \overline{I}=\Big\{i\!:\ 0\le\overline{\varrho}\|x_{\!_{J_i}}\|_p<\frac{1}{1-t^*}\Big\}.
  \]
  By using Lemma \ref{lemma-phi} of Appendix A with $\omega=\|x_{\!_{J_i}}\|_p$,
  there exists $t_0\in[0,1)$ such that
  \begin{align*}
  &{\textstyle\sum_{i=1}^m}\phi(w_i)+\overline{\varrho}\langle e\!-\!w,G_{\!\mathcal{J}\!,p}(x)\rangle
  ={\textstyle\sum_{i=1}^m}\big[\phi(w_i)+\overline{\varrho}(1-w_i)\|x_{\!_{J_i}}\|_p\big]\\
  &\ge\|y(x,\overline{\varrho})\|_0+\frac{\overline{\varrho}(1\!-t_0)}{\phi_{-}'(1)(1\!-t^*)}
  {\textstyle\sum_{i\in I}}\|x_{\!_{J_i}}\|_p +\overline{\varrho}(1\!-t_0){\textstyle\sum_{i\in \overline{I}}}\|x_{\!_{J_i}}\|_p\\
  &\ge\|y(x,\overline{\varrho})\|_0+\beta\nu L_{\!f}{\textstyle\sum_{i\in I\cup\overline{I}}}\,\|x_{\!_{J_i}}\|_p
  \ge\|y(x,\overline{\varrho})\|_0+\nu L_{\!f}\|x\!-\!x^{\overline{\varrho}}\|_2\nonumber\\
  &\ge \|G_{\!\mathcal{J}\!,p}(x^{\overline{\varrho}})\|_0+\nu(f(x^{\overline{\varrho}})-f(x))
 \end{align*}
  where the first inequality is using the definition of $y(x,\overline{\varrho})$,
  the second one is using $0=\phi(t^*)\ge\phi(1)+\phi_{-}'(1)(t^*-1)$ by
  the convexity of $\phi$ in $[0,1]$, and the last one is due to the Lipschitz
  of $f$ relative to $\Omega$ and $x^{\overline{\varrho}}\in\Omega$.
  From the last inequality and $x^{\overline{\varrho}}\in\Omega$, we have
  \[
   \nu f(x)+{\textstyle\sum_{i=1}^m}\phi(w_i)+\!\overline{\varrho}\langle e\!-\!w,G_{\!\mathcal{J}\!,p}(x)\rangle
   \ge\|G_{\!\mathcal{J}\!,p}(x^{\overline{\varrho}})\|_0+\nu f(x^{\overline{\varrho}})\ge\varpi^*.
  \]
  This, by the arbitrariness of $(x,w)$ in $\Omega\times\![0,e]$, shows
  that the MPEC \eqref{Gzr-equiv} is uniformly partial calm
  over its optimal solution set. The proof is completed.
 \end{proof}
 \begin{remark}\label{Gzr-remark}
  When $\Omega$ takes $\mathbb{R}^n,\,\mathbb{R}_{+}^n$,
  $\{x\in\mathbb{R}^n\ |\ \|x\|_p\le \gamma\}$ or
  $\{x\in\mathbb{R}_{+}^n\ |\ \|x\|_p\le \gamma\}$ for some $\gamma>0$,
  one may check that for any $x\in\Omega$ and $\varrho\!>0$,
  the associated $x^{\varrho}$ lies in $\Omega$.
 \end{remark}
 \begin{corollary}\label{Grz-corollary}
  Let $\phi\in\Phi$. Under the assumption of Theorem \ref{Grz-theorem2},
  there exists $\varrho^*>0$ such that \eqref{Gzrmin} has the same global
  optimal solution set as \eqref{Grz-surrogate} with $\varrho\ge\varrho^*$ does:
 \begin{equation}\label{Grz-surrogate}
  \min_{x\in\mathbb{R}^n}
  \Big\{\nu f(x)+\varrho{\textstyle\sum_{i=1}^m}\|x_{\!_{J_i}}\|_p
        -{\textstyle\sum_{i=1}^m}\psi^*(\varrho\|x_{\!_{J_i}}\|_p):\ x\in\Omega\Big\}.
 \end{equation}
 \end{corollary}

 The problem \eqref{Grz-surrogate} provides a class of equivalent Lipschitz surrogates
 for the group zero-norm regularized problem \eqref{Gzrmin}. If in addition the function
 $f$ and the set $\Omega$ are convex, it also provides a class of equivalent D.C. surrogates
 for the problem \eqref{Gzrmin}.
 \section{Equivalent L-surrogates of rank optimization problems}\label{sec4}

  Low-rank optimization problems arise frequently from control and system identification,
  statistics, machine learning, signal and image processing and finance
  (see, e.g., \cite{Fazel02,Recht10,Candes09,Negahban11,Pietersz04}).
  By \cite[Lemma 3.1]{BiPan17}, for any given $X\in\mathbb{R}^{n_1\times n_2}$,
  with $\phi\in\Phi$ it holds that
 \begin{equation}\label{rank-chara}
  {\rm rank}(X)=\min_{W\in\mathbb{R}^{n_1\times n_2}}
  \!\Big\{{\textstyle\sum_{i=1}^{n_1}}\phi(\sigma_i(W))\!:\ \|X\|_*-\langle W, X\rangle=0,\|W\|\leq 1\Big\}.
 \end{equation}
  Notice that $\|X\|_*\!-\langle W,X\rangle=0,\|W\|\le 1$ iff
  $X\in\mathcal{N}_{\mathfrak{B}}(W)$, where $\mathcal{N}_{\mathfrak{B}}(W)$
  is the normal cone to $\mathfrak{B}$ at $W$ in the sense of convex analysis.
  Equation \eqref{rank-chara} shows that the rank function is actually
  an optimal value function of a parameterized equilibrium problem.

 \subsection{Rank minimization problems}\label{sec4.1}

 Let $f\!:\mathbb{R}^{n_1\times n_2}\to(-\infty,+\infty]$ be a proper lsc function,
 and let $\Omega\subseteq\mathbb{R}^{n_1\times n_2}$ be a closed set.
 Given a noise level $\delta>0$, this subsection concerns with
 the rank minimization problem
 \begin{equation}\label{rankmin}
  \min_{X\in \mathbb{R}^{n_1\times n_2}}\!\Big\{ {\rm rank}(X)\!: f(X)\le \delta,\, X\in\Omega \Big\}.
 \end{equation}
 We assume that this problem has a nonempty global optimal solution set
 and a nonzero optimal value, denoted by $r^*$. Denote by $\mathcal{S}$
 the feasible set of \eqref{rankmin}. From the variational characterization
 of the rank function in \eqref{rank-chara}, it is easy to obtain the following result.
 \begin{lemma}\label{rankmin-lemma1}
  Let $\phi\in\Phi$. The rank minimization problem \eqref{rankmin} can be reformulated as
 \begin{equation}\label{rankmin-equiv}
  \min_{X,W\in\mathbb{R}^{n_1\times n_2}}\!\Big\{{\textstyle\sum_{i=1}^{n_1}}\,\phi(\sigma_i(W))\!:
    \|X\|_*-\!\langle W,X\rangle=0,\,\|W\|\le 1,X\in \Omega,\,f(X)\le \delta\Big\}
 \end{equation}
 in the sense that if $X^*\!=U^*[{\rm Diag}(\sigma(X^*))\ \ 0](V^*)^{\mathbb{T}}$
 is a global optimal solution of \eqref{rankmin},
 then $(X^*\!,U_1^*(V_1^*)^{\mathbb{T}}\!+\!t^*U_2^*[{\rm Diag}(e)\ 0](V_2^*)^{\mathbb{T}})$
 is globally optimal to the problem \eqref{rankmin-equiv};
 and conversely, if $(X^*\!,W^*)$ is globally optimal to \eqref{rankmin-equiv},
 then $X^*$ is globally optimal to \eqref{rankmin}.
 \end{lemma}

 Lemma \ref{rankmin-lemma1} shows that the problem \eqref{rankmin}
 is equivalent to the matrix MPEC \eqref{rankmin-equiv} which,
 by the following theorem, is uniformly
 partial calm in the global optimal solution set.
 \begin{theorem}\label{theorem1-rankmin}
  Let $\phi\in\Phi$. Suppose $\mathcal{S}$ is compact. Then, the following statements hold:
  \begin{itemize}
    \item [(a)] there exists a constant $\alpha\!>0$ such that for all $X\!\in\!\mathcal{S}$,
                $\sigma_{r^*}(X)\!\ge\alpha$;

    \item [(b)] the problem \eqref{rankmin-equiv} is uniformly partial calm over its global
                optimal solution set and the following problem is a global exact penalty
                with threshold $\overline{\varrho}={\phi'_-(1)}/{\alpha}$:
                \begin{equation*}
                 \min_{X,W\in\mathbb{R}^{n_1\times n_2}}\Big\{{\textstyle\sum_{i=1}^{n_1}}\,\phi(\sigma_i(W))+\varrho(\|X\|_*\!-\!\langle W,X\rangle)\!: \|W\|\le 1,X\in \Omega,f(X)\le \delta\Big\}.
                \end{equation*}

  \end{itemize}
 \end{theorem}
 \begin{proof}
  By using the contradiction arguments as those for Lemma \ref{lemma-cond},
  one may get (a). To prove part (b), let $(X,W)$ be an arbitrary point from
  the set $\mathcal{S}\times\mathfrak{B}$. Then,
  \begin{align*}
  {\textstyle\sum_{i=1}^{n_1}} \phi(\sigma_i(W))+\overline{\varrho}\,(\|X\|_*-\langle W, X\rangle)
  &\ge{\textstyle\sum_{i=1}^{n_1}}\big[\phi(\sigma_i(W))+\overline{\varrho}\,\sigma_i(X)(1-\sigma_i(W))\big]\nonumber\\
  &\ge {\textstyle\sum_{i=1}^{r^*}}\big[\phi(\sigma_i(W))+\overline{\varrho}\,\sigma_i(X)(1-\sigma_i(W))\big]\nonumber\\
  &\ge {\textstyle\sum_{i=1}^{r^*}}\big[\phi(\sigma_i(W))+\phi'_-(1) (1-\sigma_i(W))\big]\\
  &\ge r^*\phi(1)=r^*
 \end{align*}
  where the first inequality is by the von Neumann's trace inequality,
  the third one is due to $\sigma_i(X)\ge\alpha$ for $i=1,\ldots,r^*$ and
  $\overline{\varrho}={\phi'_-(1)}/{\alpha}$, and the last one is using
  $\phi(t)\ge\phi(1)+\phi_{-}'(1)(t\!-\!1)$ for $t\in[0,1]$.
  By Lemma \ref{rankmin-lemma1}, $r^*$ is the optimal value of \eqref{rankmin-equiv}.
  Thus, by the arbitrariness of $(X,W)$ in $\mathcal{S}\times\mathfrak{B}$,
  the result follows from the last inequality.
 \end{proof}

  By the von Neumann's trace inequality and the conjugate of $\psi$,
  one may check that
  \[
   -{\textstyle\sum_{i=1}^{n_1}}\psi^*(\varrho\sigma_i(X))
   =\min_{\|W\|\le 1}\!\big\{{\textstyle\sum_{i=1}^{n_1}}\phi(\sigma_i(W))-\langle\varrho X,W\rangle\big\}.
  \]
  Together with Theorem \ref{theorem1-rankmin}(b), we immediately obtain
  the following corollary.
 \begin{corollary}\label{rankmin-corollary}
  Let $\phi\in\Phi$. Under the assumption of Theorem \ref{theorem1-rankmin},
  the problem \eqref{rankmin} has the same global optimal solution set
  as the following problem with $\rho>{\phi'_-(1)}/{\alpha}$ does:
  \begin{equation}\label{rankmin-surrogate}
   \min_{X\in\mathbb{R}^{n_1\times n_2}}\Big\{\varrho\|X\|_*\!-{\textstyle\sum_{i=1}^{n_1}}\psi^*(\varrho\sigma_i(X))\!:
    f(X)\le \delta,X\in \Omega\Big\}.
  \end{equation}
 \end{corollary}

  Since the singular value function is globally Lipschitz, the objective function
  of \eqref{rankmin-surrogate} is locally Lipschitz in $\mathbb{R}^{n_1\times n_2}$.
  Hence, the problem \eqref{rankmin-surrogate} gives a class of equivalent Lipschitz
  surrogates for the rank minimization problem \eqref{rankmin}. By Lemma \ref{Thetarho-lemma}
  in Appendix A, the objective function of \eqref{rankmin-surrogate} is actually
  the difference of two convex functions. Thus, if the feasible set of \eqref{rankmin} is convex,
  the problem \eqref{rankmin-surrogate} gives a class of equivalent D.C. surrogates.

 \subsection{Rank regularized minimization problems}\label{sec3.2}

 In this subsection, we consider the following rank regularized minimization problem
 \begin{equation}\label{rank-rmin}
  \min_{X\in\mathbb{R}^{n_1\times n_2}}\!\Big\{\nu f(X)+{\rm rank}(X)\!: X\in\Omega\Big\},
 \end{equation}
 where $f$ and $\Omega$ are same as in Subsection \ref{sec3.1}. We assume that
 \eqref{rank-rmin} has a nonempty global optimal solution set.
 From \eqref{rank-chara}, it is immediate to have the following result.
 \begin{lemma}\label{rank-rmin-lemma1}
  Let $\phi\!\in\Phi$. The rank regularized minimization problem \eqref{rank-rmin} is equivalent to
  \begin{equation}\label{rank-rmin-equiv}
  \min_{X,W\in\mathbb{R}^{n_1\times n_2}}\!\Big\{\nu f(X)+{\textstyle\sum_{i=1}^{n_1}}\phi(\sigma_i(W))\!:
    \|X\|_*-\langle W,X\rangle=0,\,\|W\|\le 1,X\in \Omega\Big\},
 \end{equation}
 in the sense that if $X^*=U^*[{\rm Diag}(\sigma(X^*))\ \ 0](V^*)^{\mathbb{T}}$ is
 a global optimal solution of \eqref{rank-rmin},
 then $(X^*,U_1^*(V_1^*)^{\mathbb{T}}\!+\!t^*U_2^*[{\rm Diag}(e)\ 0](V_2^*)^{\mathbb{T}})$
 is globally optimal to the problem \eqref{rank-rmin-equiv}; and conversely,
 if $(X^*\!,W^*)$ is globally optimal to \eqref{rank-rmin-equiv},
 then $X^*$ is globally optimal to \eqref{rank-rmin}.
 \end{lemma}

  Lemma \ref{rank-rmin-lemma1} states that the rank regularized problem \eqref{rank-rmin}
  is equivalent to the MPEC \eqref{rank-rmin-equiv}. Next we prove that under a mild restriction on $\Omega$,
  this MPEC is uniformly partial calm over its global optimal solution set.
  To this end, for any given $\varrho>0$ and $X\in\Omega$ with the SVD as
  $U[{\rm Diag}(\sigma(X))\ \ 0]V^{\mathbb{T}}$, with $\phi\in\Phi$ we define
  the matrix $X^{\varrho}\in\mathbb{R}^{n_1\times n_2}$ by
  \begin{equation}\label{Xrho}
   X^{\varrho}\!:=U[{\rm Diag}(x^{\varrho})\ \ 0]V^{\mathbb{T}}
   \ \ {\rm with}\ \
   x_i^{\varrho}:=\!\left\{\begin{array}{cl}
                  \!\sigma_i(X) & {\rm if}\ \varrho\sigma_i(X)>\phi_{-}'(1);\\
                             0 &{\rm otherwise}.
                \end{array}\right.
  \end{equation}

  \vspace{-0.5cm}
  \begin{theorem}\label{rankrmin-theorem1}
   Let $\phi\in\Phi$. Suppose that $f$ is Lipschitzian relative to $\Omega$
   with constant $L_f>0$, and that for any given $X\in\Omega$ and $\varrho>0$,
   the matrix $X^{\varrho}$ lies in $\Omega$. Then, the MPEC \eqref{rank-rmin-equiv}
   is uniformly partial calm over its global optimal solution set and
   the problem
   \begin{equation*}
   \min_{X,W\in\mathbb{R}^{n_1\times n_2}}\!\Big\{\nu f(X)+{\textstyle\sum_{i=1}^{n_1}}\phi(\sigma_i(W))
   +\varrho(\|X\|_*\!-\langle W,X\rangle)\!: X\in \Omega,\|W\|\le 1\Big\}
  \end{equation*}
   is a global exact penalty for the MPEC \eqref{rank-rmin-equiv} with threshold
   $\overline{\varrho}=\phi_{-}'(1)(1-t^*)(1-t_0)^{-1}\nu L_{\!f}$.
  \end{theorem}
  \begin{proof}
  Take an arbitrary point $(X,W)$ from the set $\Omega\times\mathfrak{B}$.
  Let $X$ have the SVD given by $U[{\rm Diag}(\sigma(X))\ \ 0]V^{\mathbb{T}}$,
  and let $X^{\overline{\varrho}}$ be defined by \eqref{Xrho} with $\varrho=\overline{\varrho}$.
  Define the index sets
  \begin{equation}\label{indexJ}
    J=\Big\{j\!: (1-t^*)^{-1}\le\overline{\varrho}\sigma_j(X)\le\phi_{-}'(1)\Big\}
    \ \ {\rm and}\ \
    \overline{J}=\Big\{j\!: 0\le\overline{\varrho}\sigma_j(X)<(1-t^*)^{-1}\Big\}.
  \end{equation}
  By using the von Nenumann's trace inequality, it is not difficult to check that
  \begin{equation}\label{temp-prob}
    \min_{\|Z\|\le 1}\!\bigg\{\sum_{i=1}^{n_1}\phi(\sigma_i(Z))+\overline{\varrho}(\|X\|_*-\langle Z,X\rangle)\bigg\}
    =\sum_{i=1}^{n_1}\min_{t\in[0,1]}\Big\{\phi(t)+\overline{\varrho}\sigma_i(X)(1-t)\Big\},
  \end{equation}
  and hence ${\textstyle\sum_{i=1}^{n_1}}\phi(\sigma_i(W))+\overline{\varrho}\big(\|X\|_*\!-\langle W,X\rangle\big)
  \ge\sum_{i=1}^{n_1}\min_{t\in[0,1]}\big\{\phi(t)+\overline{\varrho}\sigma_i(X)(1-t)\big\}$.
  Together with Lemma \ref{lemma-phi} of Appendix A with $\omega=\sigma_i(X)$
  and the definition of $X^{\overline{\varrho}}$,
  \begin{align*}
   &{\textstyle\sum_{i=1}^{n_1}}\phi(\sigma_i(W))+\overline{\varrho}\big(\|X\|_*\!-\langle W,X\rangle\big)\\
   &\ge{\rm rank}(X^{\overline{\varrho}})+\frac{\overline{\varrho}(1\!-t_0)}{\phi_{-}'(1)(1\!-t^*)}
     {\textstyle\sum_{j\in J}}\,\sigma_j(X)+\overline{\varrho}(1-t_0){\textstyle\sum_{j\in\overline{J}}}\,\sigma_j(X)\nonumber\\
   & \ge {\rm rank}(X^{\overline{\varrho}})+\nu L_{\!f}{\textstyle\sum_{j\in J\cup\overline{J}}}\,\sigma_j(X)
   = {\rm rank}(X^{\overline{\varrho}})+\nu L_{\!f}\|X-X^{\overline{\varrho}}\|_*\\
   &\ge {\rm rank}(X^{\overline{\varrho}})+\nu L_{\!f}(f(X^{\overline{\varrho}})-f(X))
  \end{align*}
  where the second inequality is using the definition of $\overline{\varrho}$
  and $1=\phi(1)\le\phi_{-}'(1)(1\!-t^*)$,
  and the last one is due to $\|X-X^{\overline{\varrho}}\|_*\ge\|X-X^{\overline{\varrho}}\|_F$
  and the Lipschitz of $f$ relative to $\Omega$.
  Since $\nu f(X^{\overline{\varrho}})+{\rm rank}(X^{\overline{\varrho}})$
  is no less than the optimal value of \eqref{rank-rmin-equiv},
  the last inequality shows that \eqref{rank-rmin-equiv} is uniformly partial calm
  over its global optimal solution set.
  \end{proof}
 \begin{remark}\label{rankrmin-remark}
  There are many sets $\Omega$ such that the assumption of Theorem \ref{rankrmin-theorem1}
  holds, for example, $\mathbb{S}_{+}^n,
  \{X\!\in\mathbb{R}^{n_1\times n_2}\ |\ |\!\|X\|\!|\le\!\gamma\}$ and
  $\{X\!\in\mathbb{S}_{+}^{n}\ |\ |\!\|X\|\!|\le\!\gamma\}$,
  where $|\!\|\cdot\|\!|$ represents a unitarily invariant matrix norm
  and $\mathbb{S}_{+}^n$ denotes the cone consisting of all $n\times n$
  real symmetric positive semidefinite matrices.
 \end{remark}

  Theorem \ref{rankrmin-theorem1} recovers the result of \cite[Theorem 3.4]{BiPan17}
  with $\phi\in\Phi$, which includes the function family $\Phi_0$ in \cite{BiPan17} .
  In particular, comparing with the proof of \cite[Theorem 3.4]{BiPan17}, we see that
  by means of the uniform partial calmness in the solution set,
  the proof for the global exact penalty is very concise.
  Similar to the rank minimization problem case, by Theorem \ref{rankrmin-theorem1}
  and the von Neumann's trace inequality, the following result holds.
 \begin{corollary}\label{rankrmin-corollary}
  Let $\phi\in\Phi$. Under the assumption of Theorem \ref{rankrmin-theorem1},
  the problem \eqref{rank-rmin} has the same global optimal solution set
  as the following problem with $\varrho>\frac{\phi_{-}'(1)(1-t^*)\nu L_{\!f}}{1-t_0}$ does:
  \begin{equation}\label{rankrmin-surrogate}
   \min_{X\in\mathbb{R}^{n_1\times n_2}}\Big\{\nu f(X)+\varrho\|X\|_*\!
   -{\textstyle\sum_{i=1}^{n_1}}\psi^*(\varrho\sigma_i(X))\!: X\in \Omega\Big\}.
  \end{equation}
 \end{corollary}

  From the discussion after Corollary \ref{rankmin-corollary},
  the problem \eqref{rankrmin-surrogate} provides a class of equivalent
  Lipschitz surrogates for \eqref{rank-rmin}, which are also D.C. surrogates
  if $f$ and $\Omega$ are convex.

  \section{Equivalent L-surrogates of rank plus zero-norm problems}\label{sec5}

  Rank plus zero-norm optimization problems arise from noisy low-rank and sparse
  decomposition of a given matrix, which has wide applications in computer vision,
  multi-task learning, bioinformatic data analysis, covariance estimation
  and hyperspectral datacubes (see, e.g., \cite{Zhou10,GV11,GPX11}).
  By the variational characterization of the zero-norm and rank function
  in \eqref{Gzn-chara} and \eqref{rank-chara}, respectively, for any given
  $X,Y\in\mathbb{R}^{n_1\times n_2}$ and $\lambda>0$, with $\phi\in\Phi$
 \begin{align}\label{rank-znorm-chara}
    {\rm rank}(X)+\lambda\|Y\|_0
    &=\min_{W\in\mathfrak{B},S\in\mathscr{B}}{\textstyle\sum_{i=1}^{n_1}}\phi(\sigma_i(W))
    +\lambda{\textstyle\sum_{i=1}^{n_1}\!\sum_{j=1}^{n_2}}\phi(|S_{ij}|)\nonumber\\
    &\qquad\quad{\rm s.t.}\ \ \|X\|_*\!-\!\langle W,X\rangle+\lambda(\|Y\|_1\!-\!\langle Y,S\rangle)=0.
  \end{align}
  Notice that $\|X\|_*\!-\!\langle W,X\rangle+\lambda(\|Y\|_1\!-\!\langle Y,S\rangle)=0,W\in\mathfrak{B},S\in\mathscr{B}$
  if and only if
  \[
    \|X\|_*\!-\!\langle W,X\rangle=0\ {\rm and}\ \|Y\|_1\!-\!\langle Y,S\rangle=0
    \Longleftrightarrow X\in\mathcal{N}_{\mathfrak{B}}(W),S\in\mathcal{N}_{\mathscr{B}}(Y).
  \]
  Equation \eqref{rank-znorm-chara} shows that the rank plus zero-norm
  is actually an optimal value function of a parameterized equilibrium problem.
  \subsection{Rank plus zero-norm minimization problems}\label{sec5.1}

  Let $f\!:\mathbb{R}^{n_1\times n_2}\times\mathbb{R}^{n_1\times n_2}\to(-\infty,+\infty]$
  be a proper lsc function, and let $\Omega$ be a closed set in
  $\mathbb{R}^{n_1\times n_2}\times\mathbb{R}^{n_1\times n_2}$. Given a noise level $\delta>0$ and a parameter $\lambda>0$,
  this subsection focuses on the following rank plus zero-norm minimization problem
  \begin{equation}\label{rank-znorm-min}
  \min_{X,Y\in \mathbb{R}^{n_1\times n_2}}\!\Big\{{\rm rank}(X)\!+\!\lambda\|Y\|_0\!: f(X,Y)\le \delta,\, (X,Y)\in\Omega\Big\}.
  \end{equation}
  We assume that this problem has a nonempty global optimal solution set
  excluding the origin. Denote by $\mathcal{S}$ the feasible set of \eqref{rank-znorm-min}.
  By \eqref{rank-znorm-chara},
  the following result holds.
  \begin{lemma}\label{rank-znorm-lemma1}
   Let $\phi\in\Phi$. The rank plus zero-norm minimization \eqref{rank-znorm-min} is equivalent to
  \begin{align}\label{rank-znorm-equiv}
   &\min_{X,W,Y,S}{\textstyle\sum_{i=1}^{n_1}}\phi(\sigma_i(W))
    +\lambda{\textstyle\sum_{i=1}^{n_1}\!\sum_{j=1}^{n_2}}\phi(|S_{ij}|)\nonumber\\
    &\ \ \ {\rm s.t.}\ \ \|X\|_*-\langle W,X\rangle+\lambda(\|Y\|_1\!-\!\langle Y,S\rangle)=0,\\
    &\qquad\quad \|W\|\le 1,\|S\|_{\infty}\le 1,(X,Y)\in\Omega,\, f(X,Y)\le\delta\nonumber
   \end{align}
   in the sense that if $(X^*\!,Y^*)$ is a global optimal solution of the problem \eqref{rank-znorm-min},
   then $(X^*\!,W^*\!,Y^*\!,{\rm sgn}(Y^*)\!+\!t^*(E-{\rm sgn}(|Y^*|)))$ is globally optimal to
   \eqref{rank-znorm-equiv} where $W^*$ is same as in Lemma \ref{rankmin-lemma1};
   and conversely, if $(X^*\!,W^*,\!Y^*\!,S^*)$ is a global optimal solution of \eqref{rank-znorm-equiv},
   then $(X^*,Y^*)$ is globally optimal to the problem \eqref{rank-znorm-min}.
  \end{lemma}

  Lemma \ref{rank-znorm-lemma1} demonstrates that the rank plus zero-norm minimization
  problem \eqref{rank-znorm-min} is equivalent to the MPEC \eqref{rank-znorm-equiv}
  which involves two independent equilibrium constraints.
  Next we show that this MPEC is partially calm over its global optimal solution set.
 \begin{theorem}\label{rankznorm-theorem1}
  Let $\phi\in\Phi$. The problem \eqref{rank-znorm-equiv} is partially calm
  in its optimal solution set, and hence, when $\mathcal{S}$ is bounded,
  the following problem is a global exact penalty for \eqref{rank-znorm-equiv}:
  \begin{align}\label{rank-znorm-epenalty}
   &\min_{X,W,Y,S}\sum_{i=1}^{n_1}\phi(\sigma_i(W))+\lambda\!\sum_{i=1}^{n_1}\!\sum_{j=1}^{n_2}\phi(|S_{ij}|)
                  +\!\varrho\big(\|X\|_*\!-\!\langle W,X\rangle\big)+\!\varrho\lambda\big(\|Y\|_1\!-\!\langle Y,S\rangle\big)\nonumber\\
   &\quad{\rm s.t.}\ \ \|W\|\le 1,\|S\|_{\infty}\le 1,(X,Y)\in\Omega,\, f(X,Y)\le\delta.
  \end{align}
 \end{theorem}
 \begin{proof}
  Let $(X^*,W^*,Y^*,S^*)$ be an arbitrary global optimal solution of \eqref{rank-znorm-equiv}.
  By Lemma \ref{rank-znorm-lemma1}, $(X^*,Y^*)$ is globally optimal to \eqref{rank-znorm-min}
  and consequently $X^*\ne 0$ and $Y^*\ne 0$. Write $r^*={\rm rank}(X^*)$ and $s^*=\|Y^*\|_0$.
  Then $\sigma_{r^*}(X^*)>0$ and $\pi_{s^*}(Y^*)>0$. By the continuity of $\sigma_{r^*}(\cdot)$
  and $\pi_{s^*}(\cdot)$, there exists $\varepsilon>0$ such that for any
  $(X,Y)\in\mathbb{B}((X^*,Y^*),\varepsilon)$,
  \begin{equation}\label{temp-equa51}
    \sigma_{r^*}(X)\ge \alpha\ \ {\rm and}\ \ \pi_{s^*}(Y)\ge\alpha
    \ \ {\rm with}\ \alpha=\min(\sigma_{r^*}(X^*),\pi_{s^*}(Y^*))/2.
  \end{equation}
  We consider the perturbed problem of \eqref{rank-znorm-equiv} whose feasible set
  takes the following form
  \begin{align}\label{Heps}
   \mathcal{F}_\epsilon
   &:=\big\{(X,W,Y,S)\ \big|\ \|X\|_*-\langle W,X\rangle+\lambda(\|Y\|_1\!-\!\langle Y,S\rangle)=\epsilon,
    \|W\|\le 1,\nonumber\\
   &\qquad\qquad\qquad\qquad\qquad\qquad\quad \|S\|_{\infty}\le 1,(X,Y)\in\Omega,\, f(X,S)\le\delta\big\}.
   \end{align}
  Fix an arbitrary $\epsilon\in[-\varepsilon,\varepsilon]$. It suffices to consider the case $\epsilon\ge 0$.
  Let $(X,W,Y,S)$ be an arbitrary point from $\mathcal{F}_{\epsilon}\cap \mathbb{B}((X^*\!,W^*\!,Y^*\!,S^*),\varepsilon)$.
   Then, with $\overline{\varrho}={\phi'_-(1)}/{\alpha}$,
  \begin{align*}
   &{\textstyle\sum_{i=1}^{n_1}}\phi(\sigma_i(W))+\lambda{\textstyle\sum_{i=1}^{n_1}\!\sum_{j=1}^{n_2}}\phi(|S_{ij}|)
    +\overline{\varrho}\big(\|X\|_*\!-\!\langle W,X\rangle\big)+\overline{\varrho}\lambda\big(\|Y\|_1\!-\!\langle Y,S\rangle\big)\\
   &\ge{\textstyle\sum_{i=1}^{n_1}}\!\big[\phi(\sigma_i(W))+\overline{\varrho}\sigma_i(X)(1\!-\!\sigma_i(W))\big]
     +\lambda{\textstyle\sum_{i=1}^{n_1n_2}}\!\big[\phi(\pi_i(S))+\overline{\varrho}\pi_i(Y)\big(1\!-\!\pi_i(S)\big)\big]\\
   &\ge{\textstyle\sum_{i=1}^{r^*}}\big[\phi(\sigma_i(W))+\overline{\varrho}\sigma_{r^*}(X)(1\!-\!\sigma_i(W))\big]
      +\lambda{\textstyle\sum_{i=1}^{s^*}}\big[\phi(\pi_i(S))+\overline{\varrho}\pi_{s^*}(Y)(1\!-\!\pi_i(S))\big]\\
   &\ge{\textstyle\sum_{i=1}^{r^*}}\big[\phi(\sigma_i(W))+\phi'_-(1)(1\!-\!\sigma_i(W))\big]
      +\lambda{\textstyle\sum_{i=1}^{s^*}}\big[\phi(\pi_i(S))+\phi'_-(1)(1\!-\!\pi_i(S))\big]\\
   &\ge \phi(1)(r^*\!+\lambda s^*)={\rm rank}(X^*)+\lambda\|Y^*\|_0,
  \end{align*}
  where the first inequality is by the von Neumann's inequality and
  $\langle Y,S\rangle\le\langle\pi(Y),\pi(S)\rangle$, the second one is
  by the nonnegativity of $\phi$ in $[0,1]$, the third one is due to \eqref{temp-equa51}
  and $\overline{\varrho}={\phi'_-(1)}/{\alpha}$, and the last one is using
  $\phi(t)\ge\phi(1)+\phi_{-}'(1)(t\!-\!1)$ for $t\in[0,1]$.
  By Lemma \ref{rank-znorm-lemma1}, ${\rm rank}(X^*)+\lambda\|Y^*\|_0$ is exactly
  the optimal value of \eqref{rank-znorm-equiv}.
  Thus, by the arbitrariness of $\epsilon$ in $[-\varepsilon,\varepsilon]$ and
  that of $(X,W,Y,S)$ in $\mathbb{B}((X^*,W^*,Y^*,S^*),\varepsilon)\cap\mathcal{F}_{\epsilon}$,
  the last inequality shows that \eqref{rank-znorm-equiv} is partially calm at
  $(X^*,W^*,Y^*,S^*)$. By the arbitrariness of $(X^*,W^*,Y^*,S^*)$ in the global
  optimal solution set, it is partially calm in its optimal solution set.
  The second part of the conclusions follows by Proposition \ref{Epenalty-prop}(b).
 \end{proof}

  Similar to the rank minimization problem, from \eqref{rank-znorm-epenalty} one may get
  a class of equivalent Lipschitz surrogates for the rank plus zero-norm minimization
  problem \eqref{rank-znorm-min}, which are also D.C. surrogates if the feasible set
  of \eqref{rank-znorm-min} is convex. This result is stated as follows.
 \begin{corollary}\label{rank-znorm-corollary}
  Let $\phi\in\Phi$. If $\mathcal{S}$ is bounded, then there exists $\varrho^*>0$
  such that the problem \eqref{rank-znorm-min} has the same global optimal solution set as the following
  problem with $\varrho>\varrho^*$ does:
  \[
   \min_{(X,Y)\in\Omega}\bigg\{\varrho\|X\|_*\!-\!\sum_{i=1}^{n_1}\psi^*(\varrho\sigma_i(X))
                  +\!\varrho\lambda\|Y\|_1\!-\!\lambda\sum_{i=1}^{n_1}\!\sum_{j=1}^{n_2}\psi^*\big(\varrho|Y_{ij}|\big)\!:
                  f(X,Y)\le\delta\bigg\}.
  \]
 \end{corollary}
  \subsection{Rank plus zero-norm regularized problems}\label{sec5.2}

  This subsection concerns with the rank plus zero-norm
  regularized minimization problem
  \begin{equation}\label{rank-znorm-rmin}
  \min_{X,Y\in \mathbb{R}^{n_1\times n_2}}\!\Big\{\nu f(X,Y)+{\rm rank}(X)\!+\lambda\|Y\|_0\!: (X,Y)\in\Omega\Big\},
  \end{equation}
  where $f$ and $\Omega$ are same as those in Subsection \ref{sec5.1}.
  Assume that \eqref{rank-znorm-rmin} has a nonempty global optimal solution set.
  By the characterization in \eqref{rank-znorm-chara},
  the following result holds.
  \begin{lemma}\label{rank-znorm-lemma2}
   Let $\phi\in\Phi$. The rank plus zero-norm problem \eqref{rank-znorm-rmin} is equivalent to
  \begin{align}\label{rank-znorm-requiv}
   &\min_{X,W,Y,S}\nu f(X,Y)+\!{\textstyle\sum_{i=1}^{n_1}}\phi(\sigma_i(W))+\lambda{\textstyle\sum_{i=1}^{n_1}\!\sum_{j=1}^{n_2}}\phi(|S_{ij}|)\nonumber\\
    &\ \ \ {\rm s.t.}\ \ \|X\|_*-\langle W,X\rangle+\lambda(\|Y\|_1\!-\!\langle Y,S\rangle)=0,\\
    &\qquad\quad \|W\|\le 1,\|S\|_{\infty}\le 1,(X,Y)\in\Omega,\nonumber
   \end{align}
   in the sense that if $(X^*,Y^*)$ is a global optimal solution of the problem \eqref{rank-znorm-rmin},
   then $(X^*,W^*,Y^*,{\rm sgn}(Y^*)\!+\!t^*(E-{\rm sgn}(|Y^*|)))$ is globally optimal to
   \eqref{rank-znorm-requiv} with the optimal value equal to that of \eqref{rank-znorm-rmin},
   where $W^*$ is same as in Lemma \ref{rankmin-lemma1}; and conversely, if $(X^*,W^*,Y^*,S^*)$
   is globally optimal to \eqref{rank-znorm-requiv}, then $(X^*,Y^*)$ is globally optimal
   to \eqref{rank-znorm-rmin}.
  \end{lemma}

  By Lemma \ref{rank-znorm-lemma2}, the rank plus zero-norm regularized problem
  \eqref{rank-znorm-rmin} is equivalent to the MPEC \eqref{rank-znorm-requiv}.
  Next we shall show that under a suitable restriction on $\Omega$,
  this MPEC is uniformly partial clam over its global optimal solution set.
  To this end, for any given $(X,Y)\!\in\Omega$ and $\varrho\!>0$,
  with $\phi\in\!\Phi$ we define $X^{\varrho}\in\mathbb{R}^{n_1\times n_2}$ as in \eqref{Xrho}
  and $Y^{\varrho}\!\in\mathbb{R}^{n_1\times n_2}$ by
  \begin{equation}\label{Yrho}
    Y_{ij}^{\varrho}:=\left\{\begin{array}{cl}
                           \!|Y_{ij}|&{\rm if}\ \varrho|Y_{ij}|>\phi_{-}'(1);\\
                               0 &{\rm otherwise}.
                           \end{array}\right.
  \end{equation}

  \vspace{-0.3cm}
  \begin{theorem}\label{rankznorm-theorem2}
   Let $\phi\in\Phi$. Suppose that $f$ is Lipschitzian relative to $\Omega$ with
   constant $L_{\!f}>0$, and for any given $(X,Y)\in\Omega$ and $\varrho>0$,
   $(X^{\varrho},Y^{\varrho})\in\Omega$. Then, the MPEC \eqref{rank-znorm-requiv}
   is  uniformly partial calm over its global optimal solution set and the following problem
  \begin{align}\label{rank-znorm-repenalty}
   &\min_{X,W,Y,S}\nu f(X,Y)\!+\!\sum_{i=1}^{n_1}\phi(\sigma_i(W))\!+\lambda\sum_{i=1}^{n_1}\!\sum_{j=1}^{n_2}\phi(|S_{ij}|)
   \!+\!\varrho\big(\|X\|_*\!-\!\langle W,X\rangle\!+\!\lambda(\|Y\|_1\!-\!\langle Y,S\rangle)\big)\nonumber\\
    &\quad{\rm s.t.}\ \ \|W\|\le 1,\|S\|_{\infty}\le 1,(X,Y)\in\Omega
   \end{align}
  is a global exact penalty for the MPEC \eqref{rank-znorm-requiv}
   with threshold $\overline{\varrho}=\frac{\phi_{-}'(1)(1-t^*)\nu L_{\!f}}{(1-t_0)\min(1,\lambda)}$.
  \end{theorem}
  \begin{proof}
  Take an arbitrary $(X,Y,W,S)$ with $(X,Y)\in\Omega$ and $(W,S)\in\mathfrak{B}\times\mathscr{B}$.
  Let $X$ have the SVD as $U[{\rm Diag}(\sigma(X))\ \ 0]V^{\mathbb{T}}$.
  Let $X^{\overline{\varrho}}$ be defined as in \eqref{Xrho} with $\varrho=\overline{\varrho}$,
  and let $Y^{\overline{\varrho}}$ be defined by \eqref{Yrho} with $\varrho=\overline{\varrho}$.
  Let $J$ and $\overline{J}$ be defined by \eqref{indexJ}. Define
  \begin{equation}\label{indexI}
    I=\Big\{i\!: (1-t^*)^{-1}\le\overline{\varrho}\pi_i(Y)\le\phi_{-}'(1)\Big\}
    \ \ {\rm and}\ \
    \overline{I}=\Big\{i\!: 0\le\overline{\varrho}\pi_i(Y)<(1-t^*)^{-1}\Big\}.
  \end{equation}
  From the proof of Theorem \ref{rankrmin-theorem1} and
  $\langle Y,T\rangle\le\langle \pi(Y),\pi(T)\rangle$, it is easy to verify that
  \begin{align}\label{temp-prob1}
   &\!\min_{\|Z\|\le 1,\|T\|_{\infty}\le 1}\bigg\{\sum_{i=1}^{n_1}\phi(\sigma_i(Z))\!+\!
   \lambda\sum_{i=1}^{n_1}\!\sum_{j=1}^{n_2}\phi(|T_{ij}|)
   +\overline{\varrho}\big(\|X\|_*\!-\!\langle Z,X\rangle+\lambda(\|Y\|_1\!-\!\langle Y,T\rangle)\big)\bigg\}\nonumber\\
   &=\sum_{i=1}^{n_1}\min_{t\in[0,1]}\Big\{\phi(t)+\overline{\varrho}\sigma_i(X)(1-t)\Big\}
    +\lambda\sum_{i=1}^{n_1n_2}\min_{t\in[0,1]}\Big\{\phi(t)+\overline{\varrho}|\pi_i(Y)|(1-t)\Big\}.
  \end{align}
  Together with Lemma \ref{lemma-phi} with $\omega=\sigma_i(X)$ and $|\pi_i(Y)|$, respectively,
  it follows that
  \begin{align*}
   &{\textstyle\sum_{i=1}^{n_1}}\phi(\sigma_i(W))+\lambda{\textstyle\sum_{i=1}^{n_1}\!\sum_{j=1}^{n_2}}\phi(|S_{ij}|)
   +\overline{\varrho}\big(\|X\|_*\!-\!\langle W,X\rangle\big)
    +\overline{\varrho}\lambda\big(\|Y\|_1\!-\!\langle Y,S\rangle\big)\nonumber\\
   &\ge {\rm rank}(X^{\overline{\varrho}})+\frac{\overline{\varrho}(1\!-t_0)}{\phi_{-}'(1)(1\!-t^*)}
     {\textstyle\sum_{j\in J}}\,\sigma_j(X)+\overline{\varrho}(1-t_0){\textstyle\sum_{j\in\overline{J}}}\,\sigma_j(X)\\
   &\quad+\lambda\|Y^{\overline{\varrho}}\|_0+\frac{\lambda\overline{\varrho}(1\!-t_0)}{\phi_{-}'(1)(1\!-t^*)}
     {\textstyle\sum_{i\in I}}\,\pi_i(Y)+\lambda\overline{\varrho}(1-t_0){\textstyle\sum_{i\in\overline{I}}}\pi_i(Y)\\
   &\ge {\rm rank}(X^{\overline{\varrho}})+\nu L_{\!f}{\textstyle\sum_{j\in J\cup\overline{J}}}\,\sigma_i(X)
      +\lambda\|Y^{\overline{\varrho}}\|_0+\nu L_{\!f}{\textstyle\sum_{i\in I\cup\overline{I}}}\,\pi_i(Y)\nonumber\\
   &\ge {\rm rank}(X^{\overline{\varrho}})+\lambda\|Y^{\overline{\varrho}}\|_0
       +\nu L_{\!f}\big[\|X\!-\!X^{\overline{\varrho}}\|_*+\|Y\!-\!Y^{\overline{\varrho}}\|_1\big]\\
   &\ge {\rm rank}(X^{\overline{\varrho}})+\lambda\|Y^{\overline{\varrho}}\|_0
       + \nu f(X^{\overline{\varrho}},Y^{\overline{\varrho}})-\nu f(X,Y)
  \end{align*}
  where the second inequality is by the definition of $\overline{\varrho}$,
  and the last one is due to the Lipschitz of $f$ relative to $\Omega$ and
  $(X^{\overline{\varrho}},Y^{\overline{\varrho}})\in\Omega$.
  Since $\nu f(X^{\overline{\varrho}},Y^{\overline{\varrho}})+
  {\rm rank}(X^{\overline{\varrho}})+\lambda\|Y^{\overline{\varrho}}\|_0$ is no less than
  the optimal value of \eqref{rank-znorm-rmin}, by the arbitrariness of $((X,Y),W,S)$
  in $\Omega\times\mathfrak{B}\times\mathscr{B}$, this shows that
  \eqref{rank-znorm-requiv} is uniformly partial calm over its global optimal solution set.
  \end{proof}
 \begin{remark}\label{remark-rank-znorm-rmin}
  When $\Omega=\Omega_1\times\Omega_2$ where $\Omega_1$ takes one of the sets
  in Remark \ref{rankrmin-remark} and $\Omega_2$ takes one of the sets as in
  Remark \ref{Gzr-remark}, the assumption of Theorem \ref{rankznorm-theorem2} holds.
 \end{remark}

  Similarly, from \eqref{rank-znorm-repenalty} one may obtain a class of equivalent
  Lipschitz surrogates for \eqref{rank-znorm-rmin},
  which are also D.C. surrogates if the function $f$ and the set $\Omega$ are convex.
 \begin{corollary}\label{rank-znorm-corollary2}
  Let $\phi\in\Phi$. Under the assumption of Theorem \ref{rankznorm-theorem2},
  the problem \eqref{rank-znorm-rmin} has the same global optimal solution set
  as the following problem with $\varrho>\frac{\phi_{-}'(1)(1-t^*)\nu L_{\!f}}{\min(1,\lambda)(1\!-t_0)}$ does:
  \begin{equation}\label{rank-znorm-surrogate2}
   \min_{(X,Y)\in\Omega}\!\bigg\{\nu\!f(X,Y)+\!\varrho\|X\|_*\!-\!\sum_{i=1}^{n_1}\psi^*(\varrho\sigma_i(X))
                  +\!\varrho\lambda\|Y\|_1\!-\lambda\sum_{i=1}^{n_1}\!\sum_{j=1}^{n_2}\psi^*\big(\varrho|Y_{ij}|\big)\bigg\}.
  \end{equation}
 \end{corollary}
  \section{Equivalent L-surrogates for simultaneous rank and zero-norm minimization problems}\label{sec6}

  Simultaneous rank and zero-norm optimization problems arise from
  the applications of simultaneous structured models in signal processing,
  phase retrieval, multi-task learning and sparse principal component analysis
  (see, e.g., \cite{Richard12,OH15}). Let $f\!:\mathbb{R}^{n_1\times n_2}\to(-\infty,+\infty]$
  be a proper lsc function, and $\Omega\subseteq\mathbb{R}^{n_1\times n_2}$
  be a closed set. Given a noise level $\delta>0$ and a parameter $\lambda>0$,
  consider the simultaneous rank and zero-norm minimization
  \begin{equation}\label{srank-znorm-min}
  \min_{X\in \mathbb{R}^{n_1\times n_2}}\!\Big\{{\rm rank}(X)\!+\lambda\|X\|_0\!: f(X)\le \delta,\,X\in\Omega\Big\}.
  \end{equation}
  We assume that \eqref{srank-znorm-min} has a nonempty global optimal solution set
  excluding the origin. By the characterization in \eqref{Gzn-chara} and
  \eqref{rank-chara}, for any $X\!\in\mathbb{R}^{n_1\times n_2}$, with $\phi\in\Phi$ one has that
  \begin{align}\label{srank-znorm-chara}
    {\rm rank}(X)+\lambda\|X\|_0
    &=\min_{W\in\mathfrak{B},S\in\mathscr{B}}\sum_{i=1}^{n_1}\phi(\sigma_i(W))
    +\lambda\sum_{i=1}^{n_1}\!\sum_{j=1}^{n_2}\phi(|S_{ij}|)\nonumber\\
    &\qquad\quad{\rm s.t.}\ \ \|X\|_*\!-\!\langle W,X\rangle+\lambda(\|X\|_1\!-\!\langle X,S\rangle)=0.
  \end{align}
  Notice that $\|X\|_*\!-\!\langle W,X\rangle+\lambda(\|X\|_1\!-\!\langle X,S\rangle)=0,
  W\in\mathfrak{B},S\in\mathscr{B}$
  if and only if
  \[
   \|X\|_*-\langle W,X\rangle+\|X\|_1\!-\langle X,S\rangle=0,\|W\|\le 1,\|S\|_{\infty}\le 1
   \Longleftrightarrow X\in\mathcal{N}_{\mathfrak{B}}(W)\cap\mathcal{N}_{\mathscr{B}}(S).
  \]
  Equation \eqref{srank-znorm-chara} shows that the simultaneous rank and zero-norm function
  is also an optimal value function of a parameterized equilibrium problem.
  \begin{lemma}\label{srank-znorm-lemma1}
   Let $\phi\in \Phi$. The simultaneously structured problem \eqref{srank-znorm-min} is equivalent to
   \begin{align}\label{srank-znorm-equiv}
   &\min_{X,W,S\in\mathbb{R}^{n_1\times n_2}}\!{\textstyle\sum_{i=1}^{n_1}}\phi(\sigma_i(W))+\lambda{\textstyle\sum_{i=1}^{n_1}\!\sum_{j=1}^{n_2}}\phi(|S_{ij}|)\nonumber\\
    &\qquad\ {\rm s.t.}\ \ \|X\|_*-\langle W,X\rangle+\lambda(\|X\|_1\!-\!\langle X,S\rangle)=0,\\
    &\qquad\qquad\ \|W\|\le 1,\|S\|_{\infty}\le 1,X\in\Omega,\, f(X)\le\delta\nonumber
   \end{align}
   in the sense that if $X^*$ is a global optimal solution to \eqref{srank-znorm-min},
   then $(X^*,W^*,S^*)$ is globally optimal to \eqref{srank-znorm-equiv} where
   $W^*$ is same as in Lemma \ref{rankmin-lemma1} and $S^*\!={\rm sgn}(X^*)+t^*(E-{\rm sgn}(|X^*|))$;
   conversely, if $(X^*,W^*,S^*\!)$ is globally optimal to \eqref{srank-znorm-equiv},
   then $X^*$ is globally optimal to \eqref{srank-znorm-min}.
  \end{lemma}

  Lemma \ref{srank-znorm-lemma1} states that the simultaneous rank and zero-norm minimization
  problem \eqref{srank-znorm-min} is equivalent to the MPEC \eqref{srank-znorm-equiv},
  which involves a double equilibrium restriction on $X$.
  By following the arguments as those for Theorem \ref{rankznorm-theorem1},
  one can establish the uniform partial calmness of the MPEC \eqref{srank-znorm-equiv}
  over its global optimal solution set as follows.
  \begin{theorem}\label{srankmin-theorem1}
  Let $\phi\in\!\Phi$. The problem \eqref{srank-znorm-equiv} is partially calm
  in its optimal solution set, and if the feasible set of \eqref{srank-znorm-min} is bounded,
  the following problem is its global exact penalty:
  \begin{align}\label{srank-znorm-epenalty}
   &\min_{X,W,S}\,\sum_{i=1}^{n_1}\phi(\sigma_i(W))\!+\lambda\sum_{i=1}^{n_1}\sum_{j=1}^{n_2}\phi(|S_{ij}|)
    +\varrho\big(\|X\|_*\!-\langle W,X\rangle\big)+\varrho\lambda\big(\|X\|_1\!-\!\langle X,S\rangle\big)\nonumber\\
   &\ \ {\rm s.t.}\ \ \|W\|\le 1,\|S\|_{\infty}\le 1,X\in\Omega,\, f(X)\le\delta.
   \end{align}
 \end{theorem}

  Similarly, from \eqref{srank-znorm-epenalty} one may obtain a class of equivalent
  Lipschitz surrogates for the problem \eqref{srank-znorm-min},
  which are also D.C. surrogates if the feasible set of \eqref{srank-znorm-min} is convex.
 \begin{corollary}\label{rank-znorm-corollary}
  Let $\phi\in\Phi$. If the feasible set of \eqref{srank-znorm-min} is bounded,
  there exists $\varrho^*$ such that \eqref{srank-znorm-min} has the same
  global optimal solution set as the following problem with $\varrho>\!\varrho^*$ does:
  \[
   \min_{X\in\mathbb{R}^{n_1\times n_2}}\!\bigg\{\varrho\|X\|_*\!-\sum_{i=1}^{n_1}\psi^*(\varrho\sigma_i(X))
                  +\!\varrho\|X\|_1\!-\!\lambda\sum_{i=1}^{n_1}\!\sum_{j=1}^{n_2}\psi^*\big(\varrho|X_{ij}|\big)\!:
                  f(X)\le\delta,X\in\Omega\bigg\}.
  \]
 \end{corollary}
  \begin{remark}
  For the simultaneous rank and zero-norm regularized minimization problem
  \begin{equation}\label{srank-znorm-rmin}
  \min_{X\in\Omega}\,\Big\{\nu f(X)+{\rm rank}(X)+\lambda\|X\|_0\Big\},
  \end{equation}
  by \eqref{srank-znorm-chara} we can obtain its equivalent MPEC, but
  our current analysis technique for its uniform partial calmness
  over the global optimal solution set requires a strong restriction
  on $\Omega$. So, we do not include the result here,
  and leave it for a future research topic.
  \end{remark}
  \section{Application to low-rank plus sparsity decomposition}\label{sec7}

  We have proposed a mechanism to produce equivalent Lipschitz surrogates
  for several classes of zero-norm and rank optimization problems by
  the global exact penalty for their equivalent MPECs. This section provides
  an application of these surrogates by designing a multi-stage
  convex relaxation approach to the rank plus zero-norm regularized problem:
  \begin{equation}\label{rank-znorm-rmin1}
   \min_{X,Y\in \mathbb{R}^{n_1\times n_2}}\!\Big\{\frac{\nu}{2}\|X+Y-M\|_F^2
   +{\rm rank}(X)\!+\lambda\|Y\|_0\!: \|X\|\le\gamma_1,\|Y\|_\infty\le\gamma_2\Big\}
  \end{equation}
  where $M\in\mathbb{R}^{n_1\times n_2}$ is a given matrix,
  and $\gamma_1,\gamma_2>0$ are constants. By Corollary \ref{rank-znorm-corollary2},
  to achieve a desirable solution to \eqref{rank-znorm-rmin1},
  it suffices to design a convex relaxation approach to
  \begin{equation}\label{rank-znorm-repenalty1}
   \!\min_{\|X\|\le\gamma_1\atop\|Y\|_\infty\le\gamma_2}\!\bigg\{\frac{\nu}{2}\big\|X+Y\!-\!M\big\|_F^2+\varrho\|X\|_*-
       \!\sum_{i=1}^{n_1}\psi^*(\varrho\sigma_i(X))\!
       +\varrho\lambda\|Y\|_1-\lambda\sum_{i=1}^{n_1}\!\sum_{j=1}^{n_2}\psi^*(\varrho|Y_{ij}|)\bigg\}.
  \end{equation}
  Let $(X^{k},Y^{k})$ be the current iterate. By Lemma \ref{Thetarho-lemma},
  ${\textstyle\sum_{i=1}^{n_1}}\psi^*(\varrho\sigma_i(X))=(\widehat{\Psi}^*\circ\sigma)(\varrho X)$
  is a convex function where the expression of $\widehat{\Psi}^*$ is given by \eqref{Psihat}.
  Take $W^{k}\in\partial(\widehat{\Psi}^*\circ\sigma)(\varrho X^{k})$.
  Then, from the convexity of $\widehat{\Psi}^*\circ\sigma$,
  for any $X\in\mathbb{R}^{n_1\times n_2}$ it holds that
  \[
    {\textstyle\sum_{i=1}^{n_1}}\psi^*(\varrho\sigma_i(X))
     \ge {\textstyle\sum_{i=1}^{n_1}}\psi^*(\varrho\sigma_i(X^k))
     +\varrho\langle W^{k},X-X^{k}\rangle.
  \]
  Let $S_{ij}^{k}\in\partial\psi^*(\varrho|Y_{ij}^{k}|)$ for $i=1,\ldots,n_1$
  and $j=1,\ldots,n_2$. For any $Y\in\mathbb{R}^{n_1\times n_2}$, we have
  \[
     \sum_{i=1}^{n_1}\!\sum_{j=1}^{n_2}\psi^*(\varrho|Y_{ij}|)
     \ge\sum_{i=1}^{n_1}\!\sum_{j=1}^{n_2}\psi^*(\varrho|Y_{ij}^{k}|)
     +\varrho\sum_{i=1}^{n_1}\!\sum_{j=1}^{n_2}S_{ij}^{k}(|Y_{ij}|-|Y_{ij}^{k}|).
  \]
  The last two equations provide a global convex minorization for
  the objective function of \eqref{rank-znorm-repenalty1} at $(X^{k},Y^{k})$.
  We design a multi-stage convex relaxation approach (MSCRA)
  to \eqref{rank-znorm-rmin1} by minimizing this convex minorization.
  The iterates of the MSCRA are as follows.
 \begin{algorithm}[H]
 \caption{\label{MSCRA}{\bf\ GEP-MSCRA for solving the problem \eqref{rank-znorm-repenalty1}}}
 \textbf{Initialization:} Set the starting point $(X^0,Y^0)=(0,0)$. Select
 $W^0\in\partial(\widehat{\Psi}^*\circ\sigma)(0)$ and
 $S_{ij}^0\in\partial\psi^*(0)$ for $i=1,\ldots,n_1;j=1,\ldots,n_2$.
 Choose $\lambda_1>0$ and $\mu_1>0$. Set $k:=1$.

 \medskip
 \noindent
 \textbf{while} the stopping conditions are not satisfied \textbf{do}
 \begin{enumerate}
  \item  Seek an optimal solution $(X^{k},Y^{k})$ to the following matrix convex optimization
         \begin{equation}\label{subproblem1}
          \!\min_{\|X\|\le\gamma_1,\|Y\|_\infty\le\gamma_2}\!\Big\{\frac{1}{2}\big\|X+Y\!-\!M\big\|_F^2
          +\lambda_k\big(\|X\|_*-\!\langle W^{k-1},X\rangle\big)+\mu_k\langle E-S^{k-1},|Y|\rangle\Big\}.
         \end{equation}

  \item  When $k=1$, select $\varrho_1$ and $\widetilde{\varrho}_1$
         by the information of $\|X^1\|$ and $\|Y^1\|_\infty$. Otherwise, when $k\ge 2$,
         select $\varrho_k$ and $\widetilde{\varrho}_k$ such that $\varrho_k\ge\varrho_{k-1}$
         and $\widetilde{\varrho}_k\ge\widetilde{\varrho}_{k-1}$.

  \item  Select $W^k\in\partial(\widehat{\Psi}^*\circ\sigma)(\varrho_kX^k)$ and
         $S_{ij}^k\in\partial\psi^*(\widetilde{\varrho}_k|Y_{ij}^k|)$
         for $i=1,\ldots,n_1;j=1,\ldots,n_2$.

  \item Update $\lambda_{k}$ and $\mu_{k}$. Set~$k\leftarrow k+1$, and then go to Step 1.
 \end{enumerate}
 \textbf{end while}
 \end{algorithm}
 \begin{remark}
  Let $X^{k}$ have the SVD as $U^{k}{\rm Diag}(\sigma(X^{k}))(V^{k})^{\mathbb{T}}$.
  By \cite[Corollary 2.5]{Lewis95},
  \[
    W^{k}\in\Big\{U^{k}{\rm Diag}(w^{k})(V^{k})^{\mathbb{T}}\ |\
    w^{k}\in\partial\widehat{\Psi}^*(\varrho\sigma(X^{k}))\Big\}.
  \]
  Together with the expression of $\widehat{\Psi}^*$ in \eqref{Psihat},
  $w_i^k\in\partial(\psi^*\circ|\cdot|)(\varrho\sigma_i(X^{k}))$
  for $i=1,\ldots,n_1$. Since $\psi^*\circ|\cdot|$ and $\psi^*$ are
  convex functions on $\mathbb{R}$, it is easy to obtain $w_i^k$
  and $S_{ij}^k$; for example, when $\phi$ is the function in Example \ref{SCAD}
  of Appendix B, it holds that
  \begin{subnumcases}{}\label{wik}
   w_i^k=\min\Big[1,\max\Big(\frac{(a+1)\varrho_k\sigma_i(X^k)-2}{2(a-1)},0\Big)\Big]
   \ \ {\rm for}\ \ i=1,2,\ldots,n,\\
   S_{ij}^k=\min\Big[1,\max\Big(\frac{(a+1)\widetilde{\varrho}_k |Y_{ij}^k|-2}{2(a-1)},0\Big)\Big]
   \ {\rm for}\ i=1,\ldots,n_1;j=1,\ldots,n_2.
 \end{subnumcases}
  This means that each step of Algorithm \ref{MSCRA} is solving
  a convex matrix programming \eqref{subproblem1}, that is,
  the GEP-MSCRA yields a desirable solution to the NP-hard problem
  by solving a series of simple convex optimization problems.
  By von Neumann's trace inequality \cite[Page 182]{Horn91},
  it is not hard to verify that Step 3 is actually solving
  \[
    \!\min_{\|W\|\le1,\|S\|_\infty\le1}\!\bigg\{
          \sum_{i=1}^{n_1}\phi(\sigma_i(W))-\varrho_k\langle X^k,W\rangle
          +\sum_{i=1}^{n_1}\sum_{j=1}^{n_2}\phi(|S_{ij}|)-\widetilde{\varrho}_k\langle Y^k,S\rangle \bigg\}.
  \]
  This shows that Algorithm \ref{MSCRA} coincides with the multi-stage
  convex relaxation approach developed in \cite{BiPan17} by solving
  the global exact penalty problem \eqref{rank-znorm-repenalty}
  in an alternating way.
 \end{remark}

  We have implemented Algorithm \ref{MSCRA} with the function $\phi$ in
  Example \ref{SCAD}, where the subproblem \eqref{subproblem1}
  is solved with the accelerated proximal gradient (APG) method
  \cite{Nesterov83,Beck09,Toh10}. For the implementation of the GEP-MSCRA
  for the zero-norm minimization and rank regularized minimization,
  the reader may refer to \cite{BiPan14,BiPan17}. All runs are performed on
  an Intel Core(TM) i7-7700HQ CUP 2.80GHz, running Windows 10 and Matlab 2015a.

  \medskip

  During the testing, we choose $\lambda_1=1$
  and $\mu_1=0.5\lambda_1/\sqrt{n}$, and update $\lambda_k$ and $\mu_k$ by
  \[
    \lambda_k=\min(\max(20,(0.45n/8)),100)\lambda_1
    \ \ {\rm and}\ \ \mu_k=\tau_k\lambda_k/\sqrt{n}\ \ {\rm for}\ k\ge 2
  \]
  where $\tau_2=0.8$ and $\tau_k=0.35$ for $k\ge 3$.
  In addition, we choose $\varrho_k$ and $\widetilde{\varrho}_k$ by
  \[
    \varrho_1=\frac{10}{\|X^1\|},\ \widetilde{\varrho}_1=\frac{10}{9\|Y^1\|_\infty},\
    \varrho_k\equiv\varrho_1\ \ {\rm and}\ \ \widetilde{\varrho}_k=\frac{10}{9}\widetilde{\varrho}_{k-1}
    \ \ {\rm for}\ k\ge 2.
  \]
  We terminate Algorithm \ref{MSCRA} when $|\|X^k+Y^k-M\|_F^2-\|X^{k-1}+Y^{k-1}-M\|_F^2|\le 0.02\|M\|_F$
  and ${\rm rank}(X^{k-j})={\rm rank}(X^{k-j-1})$ for $j=0,1,2$,
  where ${\rm rank}(X)=\sum_{i=1}^n\mathbb{I}_{\{\sigma_i(X)\ge10^{-6}\|X\|\}}$.

  \medskip

  We run a series of synthetic low-rank and sparsity decomposition problems.
  For each $(n,r,s)$ triple, where $n$ ($n=n_1=n_2$) is the matrix dimension,
  $r$ is the predetermined rank, and $s$ is the predetermined sparsity.
  We generate $M=M_R+M_S+M_0$ with $M_R=RL^{\mathbb{T}}$ in the same way as does
  in \cite{Zhou10}, where $M_0$ is a noise matrix whose entries are i.i.d $N(0,\sigma^2)$,
  and $L$ and $R$ are $n\times r$ matrices whose entries are i.i.d. $N(0,\sigma_n^2)$
  with $\sigma_n^2=10\sigma/\sqrt{n}$, and the entries of $M_S$ are independently
  distributed, each taking on value $0$ with probability $1-\rho_s$ and
  uniformly distributed in $[-5,5]$ with probability $\rho_s$. For each setting
  of parameters, we report the average errors over {\bf 10} trials.
  Recall that the constraints $\|X\|\le\gamma_1$ and $\|Y\|_\infty\le\gamma_2$
  are used to ensure that \eqref{rank-znorm-rmin1} has a bounded feasible set,
  and we find from tests that the values of $\gamma_1$ and $\gamma_2$ have no
  influence on the low rank and sparsity of solutions.
  So, we always take $\gamma_1=10\|M_R\|$
  and $\gamma_2=10\|M_S\|_\infty$.

  \medskip

  We first evaluate the performance of Algorithm \ref{MSCRA} with $\sigma=0.1$,
  $\rho_s=0.1$ and $r=0.1n$ for different $n$. We measure estimation errors using
  the root-mean-squared (RMS) error as $\|\widehat{X}-M_R\|_F/n$ and
  $\|\widehat{Y}-M_S\|_F/n$ for the low-rank component and the sparse component,
  respectively, where $(\widehat{X},\widehat{Y})$ is the output of Algorithm \ref{MSCRA}.
  Figure \ref{fig1}(a) plots the RMS error of $\widehat{X}$ and $\widehat{Y}$
  as a function of $n$, and Figure \ref{fig1}(b) plots the rank curve of $\widehat{X}$
  and the sparsity curve of $\widehat{Y}$. We see that the RMS error decreases
  as $n$ increases, and the rank of $\widehat{X}$ almost equals that of $M_R$,
  while the sparsity of $\widehat{Y}$ is lower than that of $M_S$.
  This shows that the GEP-MSCRA can yield a solution with desired low-rank
  and sparse components.


 \begin{figure}[ht]
 \setlength{\abovedisplayskip}{3pt}
 \setlength{\belowcaptionskip}{-3pt}
\begin{center}
\includegraphics[width=15cm,height=6.0cm]{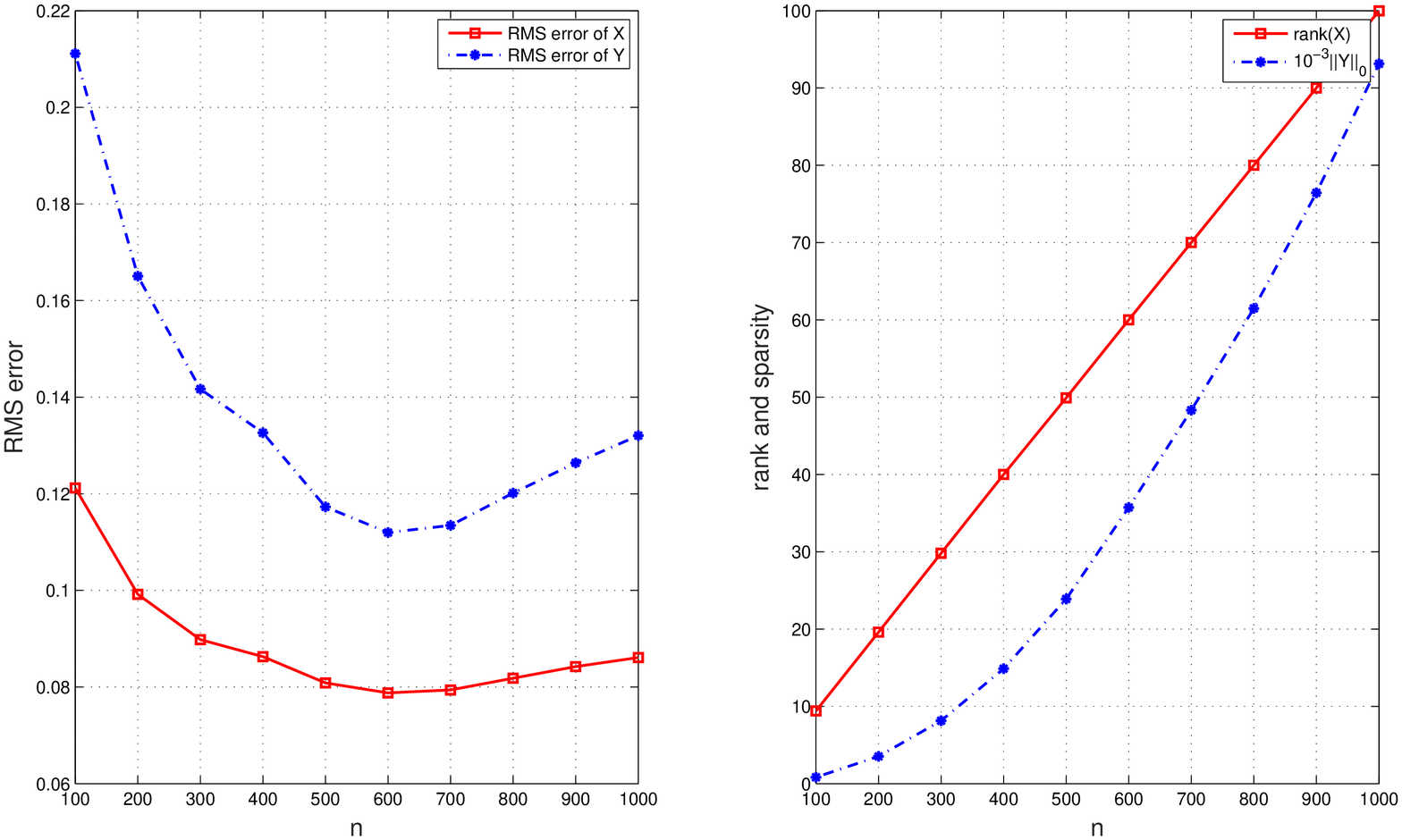}
\end{center}
\caption{RMS error as a function of $n$ in (a) and rank and sparsity in (b)}\label{fig1}
\end{figure}

\begin{figure}[ht]
\begin{center}
\includegraphics[width=15cm,height=6.0cm]{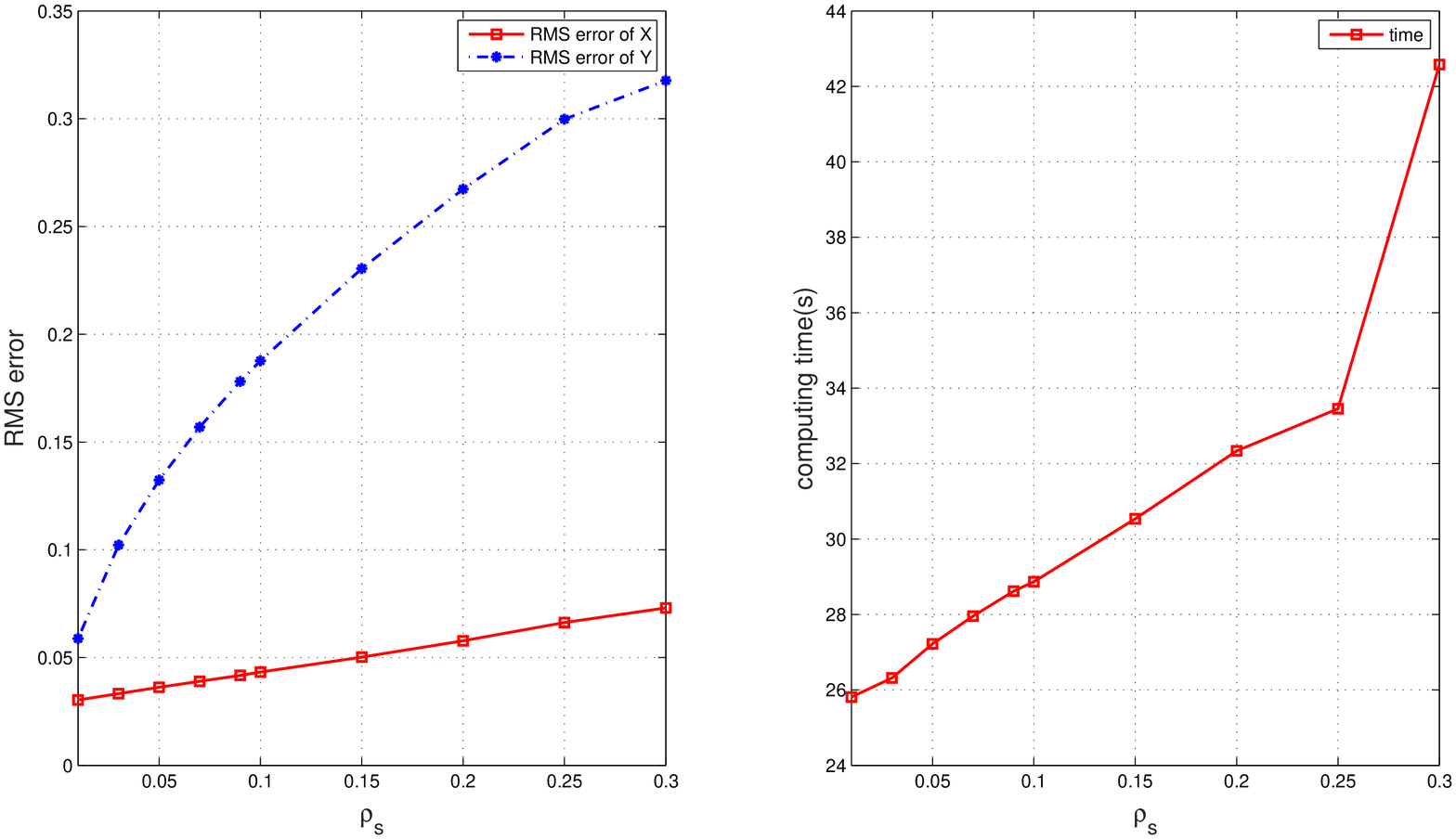}
\end{center}
\caption{RMS error as a function of $\rho_s$ in (a) and the computing time in (b)}\label{fig2}
\end{figure}

  \medskip

  Next we evaluate the performance of Algorithm \ref{MSCRA} with $\sigma=0.2$ and
  $r=10$ under different $\rho_s$. Figure \ref{fig2}(a) plots the RMS error of
  $\widehat{X}$ and $\widehat{Y}$ as a function of $\rho_s$, and Figure \ref{fig2}(b)
  plots the computing time curve. We see that the RMS error of $(\widehat{X},\widehat{Y})$
  increases as $\rho_s$ increases, and the increased range of $\widehat{X}$
  is much less than that of $\widehat{Y}$. All test problems are solved within
  one minute, and as $\rho_s$ increases, the computing time has a little increase.

 \section{Conclusion}

  We have proposed a mechanism to produce equivalent Lipschitz surrogates for
  several classes of zero-norm and rank optimization problems by the global
  exact penalty for their equivalent MPECs, and taken the rank plus zero-norm
  regularized minimization problem for example to illustrate an application of
  these surrogates in the design of multi-stage stage convex relaxation approach.
  We found that this relaxation method coincides with the one \cite{BiPan17}
  developed by solving the global exact penalty in an alternating way.

  \medskip
  \noindent
  {\bf Acknowledgement} The authors would like to thank the referee for
  his/her helpful comments, which led to significant improvements in
  the presentation of this paper.

 \bigskip
 \noindent
 {\bf\large Appendix A}

 \medskip

 The following lemma characterizes an important property of the function family $\Phi$.
 \begin{alemma}\label{lemma-phi}
  Let $\phi\in\Phi$. Then, there exists $t_0\in[0,1)$ such that
  $\frac{1}{1-t^*}\in\partial\phi(t_0)$, and
  for any given $\omega\ge 0$ and $\varrho>0$, the optimal value
  $\upsilon^*\!:=\min_{t\in[0,1]}\{\phi(t)+\varrho\omega(1-\!t)\}$ satisfies
  \[
    \left\{\begin{array}{ll}
    \upsilon^*=1&{\rm if}\ \varrho\omega\in(\phi_{-}'(1),+\infty);\\
    \upsilon^*\ge\frac{\varrho\omega(1-t_0)}{\phi_{-}'(1)(1-t^*)}&{\rm if}\  \varrho\omega\in\big[\frac{1}{1-t^*},\phi_{-}'(1)\big];\\
    \upsilon^*\ge\varrho\omega(1\!-\!t_0) &{\rm if}\  \varrho\omega\in\big[0,\frac{1}{1-t^*}\big).
    \end{array}\right.
  \]
 \end{alemma}
 \begin{proof}
  If $\phi'(t)$ is a constant for $t\in[t^*,1]$, then $\phi'(t)=\frac{\phi(1)-\phi(t^*)}{1-t^*}=\frac{1}{1-t^*}$
  for all $t\in[t^*,1]$, which means that any $t_0\in[t^*,1)$ satisfies the requirement.
  Otherwise, there must exist a point $\overline{t}\in[t^*,1)$ such that
  $\phi_{-}'(\overline{t})<\phi_{-}'(1)$. Together with the convexity of $\phi$ in $[0,1]$,
  we have $\phi_{-}'(t)\le\phi_{-}'(\overline{t})$ for $t\in[t^*,\overline{t}]$.
  By \cite[Corollary 24.2.1]{Roc70}, it follows that
  \begin{align*}
   1&=\phi(1)-\phi(t^*)
   =\int_{t^*}^{1}\phi_{-}'(t)dt
   =\int_{t^*}^{\overline{t}}\phi_{-}'(t)dt+\int_{\overline{t}}^{1}\phi_{-}'(t)dt\\
   &<\phi_{-}'(1)(\overline{t}-t^*)+\int_{\overline{t}}^{1}\phi_{-}'(t)dt
   \le \phi_{-}'(1)(1-t^*).
  \end{align*}
  Also, by the convexity of $\phi$ in $[0,1]$,
  $1=\phi(1)\ge\phi(t^*)+\phi_{+}'(t^*)(1-t^*)=\phi_{+}'(t^*)(1-t^*)$.
  Thus, $a:=\frac{1}{1-t^*}\in(\phi_{+}'(t^*),\phi_{-}'(1))\subseteq[\phi_{+}'(0),\phi_{-}'(1)]$,
  which further implies that
  \[
    (\partial\phi)^{-1}(a)=(\partial\psi)^{-1}(a)=\partial\psi^*(a)\subseteq[0,1].
  \]
  Notice that $(\partial\phi)^{-1}(a)\cap[0,1)\ne\emptyset$ (if not,
  $a\in\partial\phi(1)=[\phi_{-}'(1),\phi_{+}'(1)]$, which
  is impossible since $a<\phi_{-}'(1)$).
  Therefore, $t_0\in(\partial\phi)^{-1}(a)\cap[0,1)$ satisfies the requirement.

  \medskip

  When $\varrho\omega\ge\phi_{-}'(1)$, clearly, $\upsilon^*=\phi(1)=1$ since
  $\phi(t)+\varrho\omega(1\!-t)$ is nonincreasing in $[0,1]$.
  When $\varrho\omega\in\big[0,\frac{1}{1-t^*}\big)$,
  since $\phi_{-}'(t)\ge\phi_{+}'(t_0)>\varrho\omega$ for $t>t_0$,
  the optimal solution $\widehat{t}$ of $\min_{t\in[0,1]}\{\phi(t)+\varrho\omega(1\!-\!t)\}$
  satisfies $\widehat{t}\le t_0$. Along with the convexity of $\phi$ in $[0,1]$,
  \[
   \phi(t)+\varrho\omega(1\!-t)
   \ge \phi(\widehat{t})+\varrho\omega(1\!-\widehat{t})
   \ge \varrho\omega(1-t_0)\quad\ \forall t\in[0,1].
  \]
  This shows that $\upsilon^*\ge\varrho\omega(1-t_0)$ for this case.
  When $\varrho\omega\in\big[\frac{1}{1-t^*},\phi_{-}'(1)\big]$,
  it follows that
  \begin{equation}\label{temp-ineq}
   \phi(t)+\varrho\omega(1-t)
   \ge \phi(t)+\frac{1}{1-t^*}(1-t)\quad\ \forall t\in[0,1].
  \end{equation}
  If $t_0>0$, from the fact that $t_0<1$ and $\frac{1}{1-t^*}\in\partial\phi(t_0)$,
  it immediately follows that
  \[
    \min_{t\in[0,1]}\Big\{\phi(t)+\frac{1}{1-t^*}(1-t)\Big\}
    =\phi(t_0)+\frac{1}{1-t^*}(1-t_0)\ge\frac{1-t_0}{1-t^*}.
  \]
  Together with \eqref{temp-ineq} and $\varrho\omega\le\phi_{-}'(1)$,
  we have $\upsilon^*\ge\frac{1-t_0}{1-t^*}
  \ge\frac{\varrho\omega(1-t_0)}{\phi_{-}'(1)(1-t^*)}$.
  If $t_0=0$, from $\frac{1}{1-t^*}\in\partial\phi(t_0)$ we have
  \(
    \phi_{+}'(0)\ge\frac{1}{1-t^*}\ge 1=\phi(1)\ge\phi(0)+\phi_{+}'(0)\ge\phi_{+}'(0),
  \)
  where the third inequality is due to the convexity of $\phi$ at $[0,1]$.
  Then, for any $t\in[0,1]$,
  \[
   \phi(t)+\frac{1}{1-t^*}(1-t)
  \ge\phi(0)+\phi_{+}'(0)t+\frac{1}{1-t^*}(1-t)\ge\phi(0)+\frac{1}{1-t^*}\ge\frac{1}{1-t^*},
  \]
  where the first inequality is using the convexity of $\phi$ at $[0,1]$.
  Together with \eqref{temp-ineq}, it follows that
  $\upsilon^*\ge\frac{1}{1-t^*}\ge\frac{\varrho\omega}{\phi_{-}'(1)(1-t^*)}
  \ge\frac{\varrho\omega(1-t_0)}{\phi_{-}'(1)(1-t^*)}$.
  The proof is completed.
 \end{proof}
 \begin{alemma}\label{Thetarho-lemma}
  Let $\phi\in\Phi$. For any given $\varrho>0$, the function
  $\Theta_{\varrho}(X):={\textstyle\sum_{i=1}^{n_1}}\psi^*(\varrho\sigma_i(X))$
  is lsc and convex in $\mathbb{R}^{n_1\times n_2}$, where $\psi^*$ is the conjugate
  of $\psi$, defined by \eqref{phi-psi} with $\phi$.
 \end{alemma}
 \begin{proof}
  Let $\widehat{\Psi}(x):=\sum_{i=1}^{n_1}\widehat{\psi}(x_i)$ for $x\in\mathbb{R}^{n_1}$,
  where $\widehat{\psi}(t):=\psi(|t|)$ for $t\in\mathbb{R}$. Clearly,
  $\widehat{\Psi}$ is absolutely symmetric, i.e., $\widehat{\Psi}(x)=\widehat{\Psi}(Px)$
  for any signed permutation matrix $P\in\mathbb{R}^{n_1\times n_1}$.
  Moreover, by the definitions of $\psi$ and $\widehat{\Psi}$,
  the conjugate function $\widehat{\Psi}^*$
  of $\widehat{\Psi}$ satisfies
  \begin{align}\label{Psihat}
   \widehat{\Psi}^*(z)
   &=\sup_{x\in\mathbb{R}^{n_1}}\big\{\langle z,x\rangle-\widehat{\Psi}(x)\big\}
   ={\textstyle\sum_{i=1}^{n_1}}\sup_{t\in\mathbb{R}}\big\{z_it-\psi(|t|)\big\}\nonumber\\
   &={\textstyle\sum_{i=1}^{n_1}}\sup_{t'\in\mathbb{R}}\big\{|z_i|t'-\psi(t')\big\}
   ={\textstyle\sum_{i=1}^{n_1}}\psi^*(|z_i|)\quad{\rm for}\ z\in\mathbb{R}^{n_1}.
  \end{align}
  By the definition of $\Theta_\varrho$, we have
  $\Theta_\varrho(X)\equiv(\widehat{\Psi}^*\circ\sigma)(\varrho X)$.
  Notice that $\widehat{\Psi}^*$ is lsc and convex.
  By \cite[Lemma 2.3(b) $\&$ Corollary 2.6]{Lewis95},
  $\Theta_\varrho$ is convex and lsc on $\mathbb{R}^{n_1\times n_2}$.
 \end{proof}

  \bigskip
 \noindent
 {\bf\large Appendix B}
 \begin{aexample}\label{aexample1}
  Let $\phi(t):=t$ for $t\in \mathbb{R}$. Clearly, $\phi\in\Phi$ with $t^*=0$.
  After a calculation,
  \[
    \psi^*(s)=\left\{\begin{array}{ll}
                      s-1 & \textrm{if}\ s>1;\\
                      0 & \textrm{if}\ s\leq 1.
                \end{array}\right.
  \]
 \end{aexample}
 \begin{aexample}\label{aexample2}
  Let $\phi(t):=\frac{\varphi(t)}{\varphi(1)}$ with
  $\varphi(t)=-t-\frac{q-1}{q}(1-t+\epsilon)^{\frac{q}{q-1}}+\epsilon+\frac{q-1}{q}\,(0<q<1)$
  for $t\in (-\infty, 1+\epsilon]$, where $\epsilon\in(0,1)$ is a fixed constant.
  Now one has $\psi^*(s)=\!\frac{h(\varphi(1)s)}{\varphi(1)}$ with
  \[
   h(s):=\left\{\begin{array}{cl}
                      s+\frac{q-1}{q}\epsilon^{\frac{q}{q-1}}-\epsilon-\frac{q-1}{q} & \textrm{if}\ s\geq \epsilon^{\frac{1}{q-1}}-1,\\
                      (1\!+\epsilon)s-\frac{1}{q}(s+1)^q+\frac{1}{q} & \textrm{if}\ (1\!+\epsilon)^{\frac{1}{q-1}}\!-\!1<s<\!\epsilon^{\frac{1}{q-1}}\!-\!1,\\
                      \frac{q-1}{q}(1+\epsilon)^{\frac{q}{q-1}}-\epsilon-\frac{q-1}{q} & \textrm{if}\ s\le (1\!+\epsilon)^{\frac{1}{q-1}}-1.
                \end{array}\right.
 \]
 \end{aexample}
 \begin{aexample}\label{aexample3}
  Let $\phi(t):=\frac{\varphi(t)}{\varphi(1)}$ with
  $\varphi(t)=-t-\ln(1-t+\epsilon)+\epsilon$ for $t\in (-\infty,1+\epsilon)$,
  where $\epsilon\in(0,1)$ is a fixed constant. Clearly, $\phi\in\Phi$ with $t^*=\epsilon$.
  Now $\psi^*(s)=\frac{1}{\varphi(1)}h(\varphi(1)s)$ with
  \[
   h(s):=\left\{\begin{array}{cl}
                      s+1+\ln(\epsilon)-\epsilon & \textrm{if}\ s\geq \frac{1}{\epsilon}-1,\\
                      s(1+\epsilon)-\ln(s+1) & \textrm{if}\ \frac{1}{1+\epsilon}\!-\!1<s< \frac{1}{\epsilon}\!-\!1,\\
                      \ln(1+\epsilon)-\epsilon & \textrm{if}\ s\leq \frac{1}{1+\epsilon}-1.
                \end{array}\right.
 \]
 \end{aexample}
 \begin{aexample}\label{aexample4}
  Let $\phi(t):=\frac{\varphi(t)}{\varphi(1)}$ with
  $\varphi(t)=(1+\epsilon)\arctan\left(\sqrt{\frac{t}{1-t+\epsilon}}\right)-\sqrt{t(1-t+\epsilon)}$
  for $t\in [0,1]$, where $\epsilon\in(0,1)$ is a fixed constant.
  Now one has $\psi^*(s)=\frac{1}{\varphi(1)}h(\varphi(1)s)$ with
  \[
    h(s):=\left\{\begin{array}{cl}
                      s-(1+\epsilon)\arctan(\sqrt{1/\epsilon})+\sqrt{\epsilon} & \textrm{if}\ s\geq\!\sqrt{1/\epsilon},\\
                      s(1+\epsilon)-(1+\epsilon)\arctan(s) & \textrm{if}\ 0<s<\!\sqrt{1/\epsilon},\\
                      0 & \textrm{if}\ s\leq 0.
                \end{array}\right.
   \]
 \end{aexample}
  \begin{aexample}\label{SCAD}
   Let $\phi(t):=\frac{\varphi(t)}{\varphi(1)}$ with
   $\varphi(t)=\frac{a-1}{2}t^2+t$ for $t\in\mathbb{R}$, where $a>1$ is a fixed constant.
   Clearly, $\phi\in\Phi$ with $t^*=0$. For such $\phi$,
   one has $\psi^*(s)=\frac{1}{\varphi(1)}h(\varphi(1)s)$ with
   \[
    h(s):=\left\{\!\begin{array}{ll}
                      0 & {\rm if}\ |s|\leq 1,\\
                       \frac{(|s|-1)^2}{2(a-1)}& {\rm if}\ 1<|s|\leq a,\\
                      |s|-\frac{a+1}{2} & {\rm if}\ |s|>a.
                \end{array}\right.
  \]
  Now the objective function in \eqref{Gnz-surrogate} with $m=n$ and $J_i=\{i\}$, i.e., $\sum_{i=1}^n\big[\varrho|x_i|-\psi^*(\varrho|x_i|)\big]$
  is exactly the SCAD function \cite{FanLi01}. This shows that the minimization
  problem of the SCAD function is an equivalent surrogate for the zero-norm
  problem under a mild condition.
 \end{aexample}
\end{document}